\documentclass{article}
\PassOptionsToPackage{numbers,compress}{natbib}

 \usepackage[preprint]{nips_arxiv}

\usepackage[utf8]{inputenc}
\usepackage[T1]{fontenc}  
\usepackage{hyperref}      
\usepackage{url}         
\usepackage{booktabs}     
\usepackage{amsfonts}     
\usepackage{nicefrac} 
\usepackage{microtype}   
\usepackage{xcolor}       
\usepackage{siunitx}
\usepackage{graphicx}
\usepackage[most]{tcolorbox}
\usepackage{pifont}
\usepackage{enumitem}
\usepackage{algorithm}
\usepackage{algorithmic}
\usepackage{wrapfig}
\usepackage{booktabs}
\usepackage[table]{xcolor}
\usepackage{marvosym}
\usepackage{caption}

\newcommand{\partitle}[1]{\vspace{0.3em} \noindent \textbf{#1.}}

\usepackage{amsmath,amssymb}
\DeclareMathOperator{\erank}{erank}

\DeclareMathOperator{\op}{op}
\usepackage{amsmath,amssymb,amsthm}
\definecolor{goodgreen}{RGB}{25,135,84}
\definecolor{badred}{RGB}{200,55,55}
\definecolor{neutralgray}{RGB}{120,120,120}
\newcolumntype{C}[1]{>{\centering\let\newline\\\arraybackslash\hspace{0pt}}m{#1}}
\newcolumntype{K}[1]{>{\centering\arraybackslash}p{#1}}

\newtheorem{theorem}{Theorem}[section]

\newtheorem{lemma}[theorem]{Lemma}

\newtheorem{definition}[theorem]{Definition}
\newtheorem{assumption}[theorem]{Assumption}

\newtheorem{remark}[theorem]{Remark}
\definecolor{badred}{RGB}{200,55,55}
\definecolor{neutralgray}{RGB}{120,120,120}

\newcommand{\asrworse}[1]{\textcolor{badred}{\scriptsize $\uparrow$#1}}

\newcommand{\pplworse}[1]{\textcolor{badred}{$\uparrow$#1}}
\newcommand{\evalworse}[1]{\textcolor{badred}{$\downarrow$#1}}

\newcommand*\xbar[2][0.75]{%
    \sbox{\myboxA}{$\m@th#2$}%
    \setbox\myboxB\null%
    \ht\myboxB=\ht\myboxA%
    \dp\myboxB=\dp\myboxA%
    \wd\myboxB=#1\wd\myboxA%
    \sbox\myboxB{$\m@th\overline{\copy\myboxB}$}%
    \setlength\mylenA{\the\wd\myboxA}%
    \addtolength\mylenA{-\the\wd\myboxB}%
    \ifdim\wd\myboxB<\wd\myboxA%
        \rlap{\hskip 0.5\mylenA\usebox\myboxB}{\usebox\myboxA}%
    \else
        \hskip -0.5\mylenA\rlap{\usebox\myboxA}{\hskip 0.5\mylenA\usebox\myboxB}%
    \fi}
\makeatother

\definecolor{darkgreen}{rgb}{0.0, 0.5, 0.0}
\title{Aligned but Fragile: Enhancing LLM Safety Robustness via Zeroth-Order Optimization}

\usepackage{marvosym} 

\newcommand{\authorfootnote}[1]{%
  \begingroup
  \renewcommand\thefootnote{}%
  \footnotetext{\hspace{-1.8em}\scriptsize #1}%
  \endgroup
}

\author{
\makebox[\textwidth][c]{%
\begin{tabular}{c}
\textbf{
Zhihao Liu\textsuperscript{1,2,*},
Yifan Wu\textsuperscript{1,2,*},
Jian Lou\textsuperscript{3},
Di Wang\textsuperscript{4},
Yuxi Zhou\textsuperscript{2},
Yuke Hu\textsuperscript{4,\Letter}
}
\\[0.45em]
{\normalfont
\textsuperscript{1} The State Key Laboratory of Blockchain and Data Security, Zhejiang University
}
\\
{\normalfont
\textsuperscript{2} Hangzhou High-Tech Zone (Binjiang) Institute of Blockchain and Data Security
}
\\
{\normalfont
\textsuperscript{3} Sun Yat-sen University
\qquad
\textsuperscript{4} KAUST
}
\\[0.55em]
{\normalfont
\begin{tabular}{c}
zhihao\_liu@zju.edu.cn; gbhyck@gmail.com; jian.lou@hoiying.net \\
\{di.wang, yuke.hu\}@kaust.edu.sa; yuxizhou2025@bcds.org.cn
\end{tabular}
}
\end{tabular}%
}
}

\begin{document}

\maketitle

\authorfootnote{\textsuperscript{*} Equal contribution. \qquad \textsuperscript{\Letter} Corresponding author.}

\begin{abstract}

Safety alignment for large language models (LLMs) aims to reduce harmful or unsafe behavior while preserving general utility. However, recent findings reveal that alignment effects can be fragile: lightweight post-alignment manipulations, such as parameter noise, activation noise, or quantization, can easily weaken the intended safety behavior. Prior efforts to improve robustness have primarily focused on data curation, modified alignment objectives, and safety-critical parameter identification, leaving the role of the optimizer itself largely unexplored.

In this paper, we are the first to study the robustness of safety alignment from the perspective of the base optimizer. 
This optimizer-centric view naturally points to zeroth-order optimization, which provides a robustness-oriented signal by evaluating safety alignment under perturbations.
Based on this insight, we propose a hybrid framework that first performs standard first-order safety alignment and then applies zeroth-order refinement to improve robustness. Both theoretically and empirically, we show that only a few zeroth-order refinement steps can enhance robustness while preserving safety alignment.
We further improve the efficiency of zeroth-order refinement by exploiting its inherent perturbation-based evaluations to estimate layer-wise robustness sensitivity, enabling the refinement process to concentrate updates on robustness-critical layers with modest training overhead.

\end{abstract}

\section{Introduction}

Large language models (LLMs) \cite{touvron2023llama,qwen2offcial2025} have demonstrated remarkable capabilities across a wide range of tasks, ranging from question answering to complex instruction following. With their growing deployment in real-world settings, concerns about their reliability, safety, and robustness have become a major concern \cite{weidinger2021ethical}. To address these risks, developers introduce an additional safety alignment stage after pretraining, often through supervised fine-tuning or preference optimization. These methods mainly differ in alignment data, training objective, and model-level safety representations but are commonly optimized through gradient-based updates \cite{ouyang2022training,rafailov2023direct,ethayarajh2024kto,guan2025deliberative}. This stage makes relatively small but targeted weight adjustments to discourage the generation of harmful, unethical, or dangerous content while preserving the model's general utility and instruction-following capability \cite{gehman2020realtoxicityprompts,liu2023jailbreaking}.

Despite recent progress in developing effective and efficient LLM safety alignment algorithms, achieving \textbf{\emph{robust}} safety alignment remains a significant challenge. In particular, safety performance can deteriorate rapidly under post-alignment perturbations to model weights or intermediate activations. 
Recent studies show that even small and generic disturbances can substantially weaken safety alignment: 
injecting Gaussian noise into model weights can cause otherwise aligned models to exhibit unsafe or misaligned behavior \cite{tice2024sandbag}, 
while perturbations to specific intermediate activation layers can reliably weaken safety guardrails \cite{clymer2024poser,shahani2025noise}.
Beyond explicit noise injection, post-training compression techniques such as quantization may introduce approximation errors that weaken safety alignment~\cite{zahran2025jailbreaking}.
These perturbation-induced failures suggest that safety alignment may be encoded by a sharp or overfitted safety boundary in parameter or activation space, where small deviations can move the model from refusal to unsafe compliance. Robust alignment therefore requires not only stronger safety signals, but also smoother local safety landscapes.

Existing research on robust safety alignment has mainly focused on three directions: improving data-level robustness, redesigning objective-level alignment signals, and strengthening model-level safety representations. 
\emph{Data-level} approaches improve robustness by curating or reweighting fine-tuning data~\cite{liu2024robustifying,shen2025seal}; 
\emph{objective-level} approaches modify alignment losses to encourage more persistent refusal behavior~\cite{qi2024safety,zhao2025improving}; 
and \emph{model-level} approaches stabilize safety representations or protect safety-critical layers~\cite{huang2024vaccine,rosati2024representation,li2024safety}.
While effective, these methods typically rely on problem-specific modifications to the training data, alignment objective, or model internals. 
In contrast, the role of the \textbf{\emph{base optimizer}} itself, independent of such problem-wise and algorithm-level modifications, remains largely unexplored in shaping alignment robustness. 
The perturbation-induced failures discussed above suggest that robust alignment is also closely tied to the local geometry of the learned safety boundary, which can be shaped by the base optimizer itself.

This motivates us to revisit zeroth-order (ZO) optimization as an optimizer-level mechanism for robust safety alignment. 
Unlike standard gradient-based updates, ZO estimates update directions by evaluating losses at randomly perturbed parameters, which can be viewed as optimizing a locally smoothed objective~\cite{nesterov2017random}. 
While recent work has mainly explored ZO methods for memory-efficient LLM fine-tuning \cite{malladi2023fine,zhang2024revisiting,liu2026differentially} and LLM unlearning \cite{zhang2025towards,lang2026downgrade}, its potential for improving alignment robustness remains underexplored.

In this paper, we strive to improve the robustness of safety alignment through the lens of zeroth-order methods, which have been largely underexplored in this context. Zeroth-order optimization offers a promising route to alignment robustness as it evaluates how safety alignment behaves under perturbations. 
Our study consists of three main components. 
\ding{182} We empirically show that safety-aligned models remain fragile under simple perturbations: lightweight parameter- or activation-level noise can substantially weaken safety behavior.
\ding{183} While ZO optimization can enhance alignment robustness, directly replacing first-order (FO) alignment with ZO optimization is ineffective due to its slower and less stable convergence. Therefore, we propose a hybrid refinement strategy that applies ZO updates after standard FO safety alignment. We show theoretically and empirically that a small amount of post-alignment ZO refinement can improve safety robustness with only modest computational overhead.
\ding{184} Finally, we exploit the perturbation-based evaluations inherent to ZO optimization to estimate layer-wise robustness sensitivity, enabling a targeted refinement strategy that concentrates updates on robustness-critical layers, improving the efficiency and effectiveness of robustness-oriented safety alignment.
Our contributions can be summarized as follows:

\begin{itemize}[leftmargin=*, labelindent=0.2em]
    \item We are the first to comprehensively study the robustness of safety alignment against parameter and activation perturbations from the perspective of zeroth-order optimization.
    \item We propose an FO–ZO hybrid optimization framework that integrates FO and ZO optimizers, combining ZO-induced robustness with FO-driven alignment effectiveness.
    \item We show theoretically and empirically that a small amount of post-alignment ZO refinement can improve safety robustness with only modest computational overhead.
\end{itemize}

\section{Background and Related Work}

\subsection{Safety Alignment in Large Language Models}

Safety alignment aims to ensure that large language models (LLMs) generate outputs that are consistent with human values and safety constraints \cite{zhao2025improving,wang2025comprehensive,ji2025pku}. A standard paradigm is to post-train pretrained models using human feedback, such as Reinforcement Learning from Human Feedback (RLHF) \cite{ouyang2022training}. 
Subsequent work further improves or simplifies this paradigm from different perspectives. 
Constitutional AI reduces the dependence on direct human feedback by using a set of human-written principles and AI feedback to improve harmlessness~\cite{bai2022constitutional}; Safe RLHF explicitly separates helpfulness and harmlessness preferences and formulates safety alignment as a constrained optimization problem~\cite{dai2024safe}; Direct Preference Optimization (DPO) directly optimizes language models from preference data without explicitly training a reward model or running reinforcement learning~\cite{rafailov2023direct}; and recent reasoning-based alignment methods improve safety by training models to explicitly reason about safety requirements before generating final answers~\cite{mou2025saro}.

\subsection{Robustness of Safety Alignment}

Prior work has identified several ways to compromise the safety behavior of aligned models. At the input level, jailbreak and adversarial prompting attacks can bypass safety guardrails by rewriting, obfuscating, or optimizing harmful instructions~\cite{wei2023jailbroken,zou2023universal,shen2024anything}. At the model level, safety behavior can be weakened by malicious or even benign fine-tuning~\cite{qi2024fine,yang2023shadow}, persistent backdoors~\cite{hubinger2024sleeper}, or interventions on refusal-related representations~\cite{arditi2024refusal}. These results suggest that safety alignment may rely on mechanisms that are effective under normal conditions, yet remain fragile under adversarial inputs, distribution shifts, or post-training modifications.

While many of these failures involve stronger attack capabilities or carefully engineered procedures, safety alignment can also be vulnerable to much simpler changes. Even generic perturbations, arising from quantization, parameter noise, activation noise, hardware faults, or other post-alignment modifications, may be sufficient to degrade the model's safety behavior. Quantization reduces numerical precision to improve efficiency~\cite{dettmers2022llm}, but the induced approximation errors may alter model behavior; recent work further shows that quantization-related fault injection can weaken safety alignment~\cite{zahran2025jailbreaking}. Similarly, weight or activation perturbations can disturb internal representations and expose safety failures that are not visible under standard inference~\cite{ticenoise,clymer2024poser,shahani2025noise}.

\subsection{Zeroth-Order Optimization}

Zeroth-order (ZO) optimization decreases memory consumption by constructing gradient estimates by evaluating the objective at perturbed points in the parameter space, without requiring backpropagation~\cite{nesterov2017random}.
Beyond its computational advantage, it offers a promising route to enhancing robustness by evaluating how the model behaves under perturbations.
This intuition is further supported by a curvature-based connection: ZO optimization has been shown to favor flatter minima~\cite{zhang2025zeroth}, which contribute to improved local robustness against small perturbations~\cite{walter2025flatness}.
This perspective has motivated prior work on using ZO optimization to train and adapt large models under constrained or noisy conditions~\cite{liu2020primer,liu2026differentially}.
More recently, zeroth-order methods have also been explored in alignment-related settings. Galatolo et al.~\cite{galatolo2025visualising} explore zeroth-order methods for LLM preference optimization, aiming to reduce memory usage and iteration time.
In this work, we instead investigate zeroth-order optimization from the perspective of safety alignment robustness and show that it can enhance the resilience of LLM safety alignment against parameter and activation tampering.

\section{Preliminaries}

Zeroth-order gradient estimation approximates gradients from function evaluations at perturbed parameter points, without requiring explicit gradient information.

\begin{definition}[Zeroth-order gradient estimation \cite{spall1992multivariate}]
Given model parameters $\boldsymbol{\theta}\in\mathcal{R}^d$ and loss function $f$, the zeroth-order gradient on batch $\xi$ is estimated as 
\begin{equation}
    \nabla f_{\beta}(\boldsymbol{\theta},\xi)=\frac{1}{2\beta}\left(f(\boldsymbol{\theta}+\beta\mathbf{v},\xi)-f(\boldsymbol{\theta}-\beta\mathbf{v},\xi)\right)\mathbf{v} ,
\end{equation}
where $\mathbf{v}\sim N(0,\textbf{I}_{d})$ and $\beta$ is the ZO scale parameter.
\end{definition}

Lipschitz continuity and smoothness control the variation of the objective and its gradient, serving as standard assumptions in FO and ZO convergence analyses~\cite{malladi2023fine,liu2024zeroth}. 
We also adopt the Polyak--\L ojasiewicz (PL) condition, a common gradient-dominance assumption for deriving linear convergence-type guarantees in nonconvex settings~\cite{karimi2016linear}.

\begin{definition}[Lipschitz continuity]
The function \(f(\theta)\) is \(L\)-Lipschitz if there exists a constant \(L>0\) such that, for all $\theta,\theta'\in\mathbb{R}^d$ we have
\[
    |f(\theta)-f(\theta')|
    \le
    L\|\theta-\theta'\|.
\]
\end{definition}

\begin{definition}[Smoothness]
Let \(f(\theta)\) be differentiable. The function \(f(\theta)\) is \(h\)-smooth if its gradient is \(h\)-Lipschitz for all $\theta,\theta'\in\mathbb{R}^d$, namely,
\[
    \|\nabla f(\theta)-\nabla f(\theta')\|
    \le
    h\|\theta-\theta'\|.
\]
\end{definition}

\begin{definition}[Polyak--\L ojasiewicz condition]
The function \(f_\beta(\theta)\) satisfies the \(\mu\)-PL condition if for any $\theta\in\mathcal{R}^d$ we have
\[
\frac{1}{2}\|\nabla f_\beta(\theta)\|^2
\ge
\mu\bigl(f_\beta(\theta)-f_\beta^\star\bigr).
\]
\end{definition}

\begin{figure}[t]
    \centering
    \includegraphics[width=\linewidth]{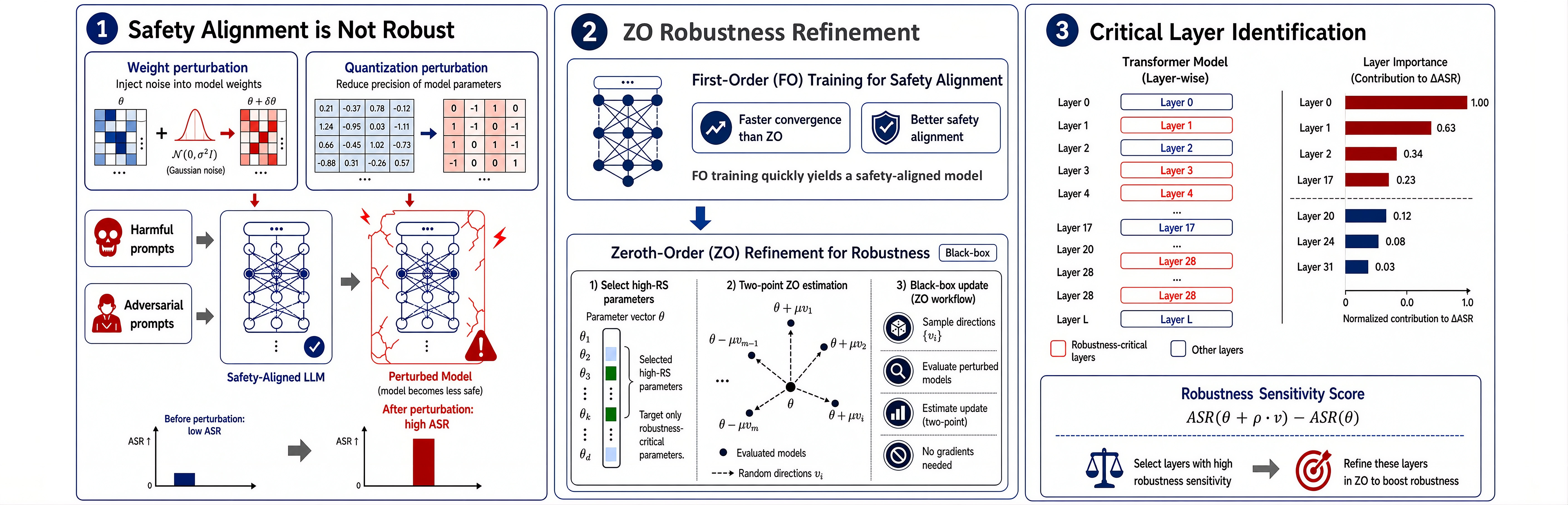}
    \caption{Overview of our robustness-oriented safety alignment framework. 
We first show that simple post-alignment perturbations can substantially increase ASR in safety-aligned LLMs. 
We then apply ZO refinement after FO safety alignment and target robustness-critical layers identified by layer-wise sensitivity scores for more efficient robustness improvement.}
    \label{fig:framework}
\end{figure}

\section{Methodology}

In this section, we present our framework for improving the robustness of safety alignment, as illustrated in Figure~\ref{fig:framework}.
Motivated by the fragility of aligned models under simple perturbations, we use the perturbation-based evaluations inherent to ZO optimization to identify robustness-critical layers and then concentrate ZO refinement on these layers.

\subsection{Zeroth-Order Robustness Refinement}
\label{subsec:zo-robustness-refinement}

We introduce ZO robustness refinement as a post-alignment stage, which improves the robustness of an already FO-aligned model rather than replacing FO safety alignment.
While FO efficiently optimizes the alignment objective, ZO is then used to enhance stability under perturbations.

\partitle{Algorithmic framework}
Algorithm~\ref{alg:fo-zo-refinement} summarizes our three-stage pipeline.
\ding{172} We first perform FO safety alignment from \(\theta_0\) to obtain an aligned model \(\theta^{\mathrm{fo}}\).
\ding{173} We then compute layer-wise robustness sensitivity scores \(S(\ell)\) and select the top-\(m\) layers as \(\mathcal{S}_{\mathrm{rob}}\); the score construction is detailed in Section~\ref{subsec:robust-param-selection}.
\ding{174} Starting from \(\theta^{\mathrm{fo}}\), we apply targeted ZO refinement only to \(\mathcal{S}_{\mathrm{rob}}\), by masking ZO perturbation directions outside the selected layers.
Thus, the pipeline first identifies robustness-critical layers and then concentrates ZO updates on the parameters most relevant to robustness degradation.
We next provide a theoretical analysis to justify the design of our FO-to-ZO pipeline.
Specifically, we aim to answer the following two questions.

\paragraph{Question 1: Why ZO after FO alignment?}
We first explain why ZO is used as a refinement stage rather than as a replacement for the entire FO alignment process.
Although ZO can optimize the smoothed objective using only function evaluations, its finite-difference estimator imposes additional dimension- and smoothing-dependent stepsize constraints, making convergence slower than FO.
Therefore, using ZO throughout the whole alignment stage would be unnecessarily expensive.
Instead, our pipeline first uses FO optimization to efficiently reach a well-aligned solution, and then applies a small number of ZO steps to refine the local robustness of this solution.
To make this distinction precise, we analyze the convergence behavior of ZO optimization on the smoothed objective \(f_\beta\).
The following result shows that ZO is a valid optimization procedure, but it requires a more restrictive stepsize than its FO counterpart.

\begin{assumption}[Local low-rank curvature upper bound]
\label{ass:low-rank-curvature}
Let \(f_\beta:\mathbb{R}^d\to\mathbb{R}\) denote the smoothed objective.
For every iterate \(\theta_t\), there exists a positive semidefinite matrix \(H_t\) such that, for all \(\theta\) on the segment between \(\theta_t\) and \(\theta_{t+1}\),
\[
    \nabla^2 f_\beta(\theta)\preceq H_t,
\]
with \(\|H_t\|_{\op}\le \ell\) and effective rank \(\erank(H_t)\le r\).
\end{assumption}

\begin{algorithm}[t]
\caption{Sensitivity-Guided ZO Robustness Refinement}
\label{alg:fo-zo-refinement}
\begin{algorithmic}[1]
\REQUIRE Initial parameters \(\theta_0=(\theta_1,\ldots,\theta_L)\), loss \(f(\theta;\xi)\), stepsizes \(\eta_{\rm fo},\eta_{\rm zo}\), ZO scale \(\beta\), FO steps \(T_{\rm fo}\), ZO steps \(T_{\rm zo}\), selected layer number \(m\).
\ENSURE Robustly refined parameters \(\theta^{\rm zo}\).

\STATE \textbf{Stage I: FO safety alignment.}
\STATE Initialize \(\theta\leftarrow\theta_0\).
\FOR{\(t=0,\ldots,T_{\rm fo}-1\)}
    \STATE Sample \(\xi_t\) and update 
    \(\theta\leftarrow\theta-\eta_{\rm fo}\nabla_\theta f(\theta;\xi_t)\).
\ENDFOR
\STATE Set \(\theta^{\rm fo}\leftarrow\theta\).

\STATE \textbf{Stage II: Robustness-aware layer selection.}
\STATE Compute layer-wise sensitivity scores \(\{S(\ell)\}_{\ell=1}^{L}\).
\STATE Select robustness-critical layers
\(\mathcal{S}_{\rm rob}=\operatorname{Top}\text{-}m\{S(\ell)\}_{\ell=1}^{L}\).

\STATE \textbf{Stage III: Targeted ZO robustness refinement.}
\STATE Initialize \(\theta\leftarrow\theta^{\rm fo}\).
\FOR{\(t=0,\ldots,T_{\rm zo}-1\)}
    \STATE Sample \(\xi_t\) and \(v_t\sim\mathcal{N}(0,I_d)\).
    \STATE Mask non-selected layers:
    \(\widetilde v_t^{(\ell)}=v_t^{(\ell)}\cdot\mathbf{1}\{\ell\in\mathcal{S}_{\rm rob}\}\).
    \STATE Evaluate \(\ell_t^{\pm}=f(\theta\pm\beta\widetilde v_t;\xi_t)\).
    \STATE Estimate 
    \(g_t^{\rm zo}=\frac{d_{\mathcal{S}}}{2\beta}(\ell_t^+-\ell_t^-)\widetilde v_t\).
    \STATE Update \(\theta\leftarrow\theta-\eta_{\rm zo}g_t^{\rm zo}\).
\ENDFOR
\STATE \textbf{return} \(\theta\).
\end{algorithmic}
\end{algorithm}

\begin{theorem}[Convergence of ZO optimization]
\label{thm:zo-global-convergence}
Suppose Assumption \ref{ass:low-rank-curvature} holds and the loss function $f(\theta)$ is $L$-Lipschitz and satisfies the $\mu$-PL condition, by setting $\eta\le\min\left\{\frac{1}{\ell r},\frac{\mu\varepsilon}{64\,\ell r\,d\,\beta^2L^4}\right\}.$ To reach $\mathbb{E}\bigl[f_\beta(\theta_t)-f_\beta^\star\bigr]\le\varepsilon$, it suffices to take
\[
    t
    \ge
    \frac{1}{\mu\eta}
    \log
    \frac{2\bigl(f_\beta(\theta_0)-f_\beta^\star\bigr)}{\varepsilon}.
\]
\end{theorem}

\begin{remark}[Why ZO refinement after FO alignment?]
The FO and ZO convergence bounds have the same iteration form, but they use different stepsizes. For FO optimization, it suffices to take
\[
    \eta_{\mathrm{fo}} < \frac{1}{\ell r},
    \qquad
    t_{\mathrm{fo}}
    \ge
    \frac{1}{\mu\eta_{\mathrm{fo}}}
    \log
    \frac{2\bigl(f(\theta_0)-f^\star\bigr)}{\varepsilon}.
\]
For ZO optimization, the stepsize must satisfy
\[
    \eta_{\mathrm{zo}}
    \le
    \min\left\{
        \frac{1}{\ell r},
        \frac{\mu\varepsilon}{64\,\ell r\,d\,\beta^2L^4}
    \right\},
    \qquad
    t_{\mathrm{zo}}
    \ge
    \frac{1}{\mu\eta_{\mathrm{zo}}}
    \log
    \frac{2\bigl(f_\beta(\theta_0)-f_\beta^\star\bigr)}{\varepsilon}.
\]
The second constraint on \(\eta_{\mathrm{zo}}\) comes from the variance and smoothing bias of the ZO estimator.
As a result, ZO usually requires more iterations than FO to reach the same optimization accuracy.
This motivates our design: FO optimization is used to efficiently obtain a safety-aligned model, while ZO is applied only after FO alignment as a robustness refinement stage.
\end{remark}

\paragraph{Question 2: How does ZO refinement improve robustness?}
The convergence result above explains why ZO is a valid post-alignment optimization procedure.
We now show why such refinement can improve robustness after FO alignment.
Let \(\theta^{\mathrm{fo}}\) be the model obtained by FO safety alignment on the objective \(f\). Although \(\theta^{\mathrm{fo}}\) may perform well on the alignment objective, it does not explicitly control the loss landscape in the surrounding parameter neighborhood.
To quantify how much safety alignment is weakened by parameter perturbations, we define the perturbation robustness gap as
\[
    \mathsf{Rob}_{\rho}(\theta)
    :=
    \mathbb{E}_{v\sim\mathcal{N}(0,I_d)}
    \bigl[f(\theta+\rho v)\bigr]
    -
    f(\theta),
\]
where \(\rho>0\) is the perturbation scale. A larger value of \(\mathsf{Rob}_{\rho}(\theta)\) means that the loss increases more significantly under parameter perturbations, indicating weaker robustness.
Even when FO alignment has already reached a stationary point of the objective \(f\), ZO refinement can still make progress on the neighborhood-averaged objective \(f_\rho\).
The following theorem formalizes this intuition.

\begin{theorem}[One-step ZO refinement improves robustness]
\label{thm:strict-zo-robustness}
Under the assumption that the loss function $f(\theta)$ is $h$-smooth and $L$-Lipschitz. Suppose \(\theta^{\mathrm{fo}}\) satisfies \(\nabla f(\theta^{\mathrm{fo}})=0\), and let \(\theta^{\mathrm{zo}}_1\) be obtained by one ZO step optimizing \(f_\rho\). For $\eta_{\mathrm{zo}}<\frac{1}{h}\cdot\frac{||\nabla f_{\rho}(\theta^{\mathrm{fo}})||^2}
{||\nabla f_{\rho}(\theta^{\mathrm{fo}})||^2+128d\beta^2L^4}$, we can have that
\[
\mathbb{E}[\mathsf{Rob}_{\rho}(\theta^{\mathrm{zo}}_1)]
<
\mathsf{Rob}_{\rho}(\theta^{\mathrm{fo}}).
\]
\end{theorem}

Theorem~\ref{thm:strict-zo-robustness} formalizes the main intuition behind our framework. Even after FO alignment has converged, the aligned model may remain non-robust in a local neighborhood of its parameters. A ZO refinement step directly optimizes this local behavior through perturbed function evaluations, and can therefore improve robustness without retraining the model from scratch. 

\begin{figure}[t]
    \centering
    \includegraphics[width=0.9\linewidth]{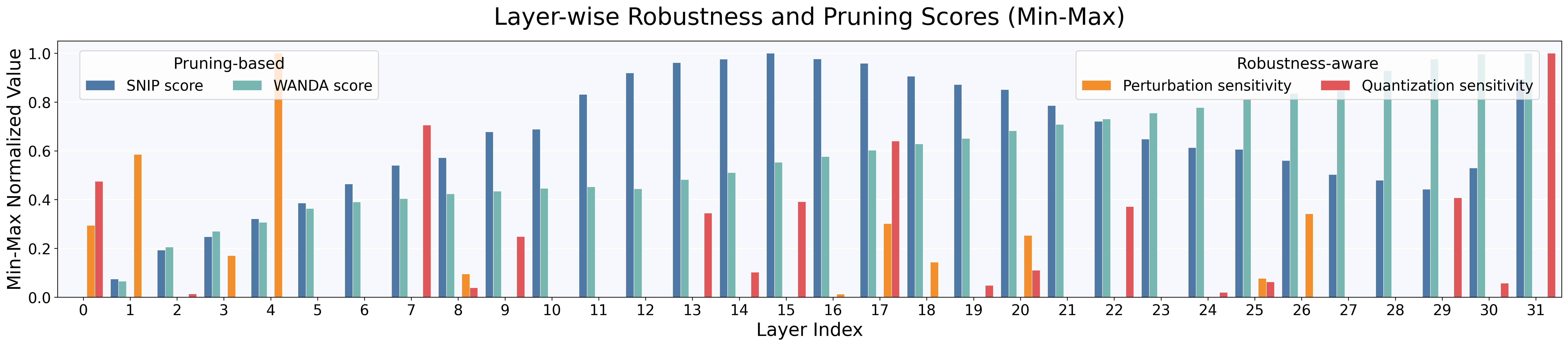}

\caption{
    Layer-wise comparison between alignment-oriented pruning scores and robustness-related sensitivity scores.
    All scores are min--max normalized for visualization.
    SNIP~\cite{leesnip} and WANDA~\cite{sun2024simple} exhibit different layer-wise patterns from perturbation- and quantization-based robustness sensitivity.
    }
    \label{fig:layer_scores}
\end{figure}

\subsection{Robustness-Aware Layer Selection}
\label{subsec:robust-param-selection}

The second stage of Algorithm~\ref{alg:fo-zo-refinement} determines which layers should be updated during targeted ZO refinement.
Rather than adopting generic pruning-based importance scores, we design a robustness-aware criterion that is directly aligned with the goal of refinement.
The key idea is simple: if perturbing a layer substantially weakens safety alignment, then this layer is more likely to be responsible for robustness degradation and should be prioritized during ZO refinement.
For each layer \(\ell\), we define a layer-wise robustness sensitivity score \(S(\ell)\) by combining two complementary perturbation effects:
\[
    S(\ell)
    =
    S_{\mathrm{noise}}(\ell)
    +
    \lambda S_{\mathrm{quant}}(\ell),
\]
where \(\lambda\) balances the contribution of quantization-based sensitivity.
The first term captures sensitivity to random local parameter perturbations, which matches the perturbation-based evaluation principle of ZO optimization.
The second term captures sensitivity to structured deployment perturbations such as quantization.
Together, \(S(\ell)\) measures how much safety alignment is weakened when layer \(\ell\) is locally or structurally perturbed.

\partitle{Random perturbation sensitivity}
For the random perturbation term, we perturb only layer \(\ell\) while keeping all other layers fixed:
\[
    \theta^{(\ell,\mathrm{noise})}
    =
    \bigl(
    \theta_1,\ldots,\theta_{\ell-1},
    \theta_\ell+\rho v_\ell,
    \theta_{\ell+1},\ldots,\theta_L
    \bigr),
    \qquad
    v_\ell \sim \mathcal{N}(0,I),
\]
where \(\rho\) is the perturbation scale.
The corresponding sensitivity is defined as
\[
    S_{\mathrm{noise}}(\ell)
    =
    f\bigl(\theta^{(\ell,\mathrm{noise})};\mathcal{D}_{\mathrm{safe}}\bigr)
    -
    f\bigl(\theta;\mathcal{D}_{\mathrm{safe}}\bigr).
\]
This term measures how strongly safety alignment is weakened by a ZO-like local perturbation applied only to layer \(\ell\).

\partitle{Quantization sensitivity}
We further include a quantization-based term to account for structured perturbations that commonly arise in practical deployment.
Let \(Q(\cdot)\) denote the quantization operator.
For each layer \(\ell\), we quantize only this layer:
\[
    \theta^{(\ell,\mathrm{quant})}
    =
    \bigl(
    \theta_1,\ldots,\theta_{\ell-1},
    Q(\theta_\ell),
    \theta_{\ell+1},\ldots,\theta_L
    \bigr).
\]
The quantization sensitivity is defined as
\[
    S_{\mathrm{quant}}(\ell)
    =
    f\bigl(\theta^{(\ell,\mathrm{quant})};\mathcal{D}_{\mathrm{safe}}\bigr)
    -
    f\bigl(\theta;\mathcal{D}_{\mathrm{safe}}\bigr).
\]
This term measures how much safety alignment is weakened when layer \(\ell\) is exposed to quantization.
After computing \(S(\ell)\) for all layers, we select the layers with the largest scores:
\[
    \mathcal{S}_{\mathrm{rob}}
    =
    \operatorname{Top}\text{-}m
    \{S(\ell)\}_{\ell=1}^{L}.
\]
The selected set \(\mathcal{S}_{\mathrm{rob}}\) is then used as the update target in the targeted ZO refinement stage.
At each refinement step, the random ZO direction is masked outside \(\mathcal{S}_{\mathrm{rob}}\), so that only robustness-critical layers are perturbed and updated.
In this way, parameter selection and ZO optimization are coupled through the same perturbation-based principle: perturbations are first used to locate vulnerable layers, and ZO updates are then concentrated on these layers to improve robustness efficiently.

\partitle{Pruning scores do not capture safety robustness}
We compare our robustness-aware scores with SNIP~\cite{leesnip} and WANDA~\cite{sun2024simple} in Figure~\ref{fig:layer_scores}.
Although pruning-based scores can indicate layers important for maintaining safety alignment, they do not consistently align with robustness sensitivity under perturbation or quantization.
This mismatch suggests that pruning scores alone are insufficient for targeted ZO refinement, as they may overlook layers that are critical to safety robustness.

\begin{table}[t]
\centering
\small
\setlength{\tabcolsep}{3.8pt}
\renewcommand{\arraystretch}{1.15}
\caption{
Safety robustness and utility of Llama-3-8B-Instruct under different perturbations.
Arrows indicate changes relative to the unperturbed original model.
Higher PPL and ASR indicate worse performance, while lower lm-eval indicates worse utility.
}
\label{tab:original_perturbation_asr}
\resizebox{\linewidth}{!}{
\begin{tabular}{llccccc}
\toprule
\textbf{Perturbation}
& \textbf{Level}
& \textbf{PPL} $\downarrow$
& \textbf{lm-eval} $\uparrow$
& \textbf{HarmBench ASR} $\downarrow$
& \textbf{LlamaGuard3 ASR} $\downarrow$
& \textbf{AdvBench ASR} $\downarrow$ \\
\midrule

Original
& --
& 8.2823
& 0.6777
& 0.1067 
& 0.0833 
& 0.1333  \\

\midrule
Quantization
& W4A16
& 9.0337 {\scriptsize(\pplworse{0.7514})}
& 0.6532 {\scriptsize(\evalworse{0.0245})}
& 0.2100 {\scriptsize(\asrworse{0.1033})}
& 0.1833 {\scriptsize(\asrworse{0.1000})}
& 0.2283 {\scriptsize(\asrworse{0.0950})} \\

Quantization
& W4A4
& 13.8641 {\scriptsize(\pplworse{5.5818})}
& 0.5603 {\scriptsize(\evalworse{0.1174})}
& 0.2567 {\scriptsize(\asrworse{0.1500})}
& 0.3700 {\scriptsize(\asrworse{0.2867})}
& 0.4133 {\scriptsize(\asrworse{0.2800})} \\

\midrule
Weight noise
& \(\sigma=1\)
& 10.4272 {\scriptsize(\pplworse{2.1449})}
& 0.6491 {\scriptsize(\evalworse{0.0286})}
& 0.1467 {\scriptsize(\asrworse{0.0400})}
& 0.1100 {\scriptsize(\asrworse{0.0267})}
& 0.1633 {\scriptsize(\asrworse{0.0300})} \\

Weight noise
& \(\sigma=1.5\)
& 13.8880 {\scriptsize(\pplworse{5.6057})}
& 0.6225 {\scriptsize(\evalworse{0.0552})}
& 0.1567 {\scriptsize(\asrworse{0.0500})}
& 0.1533 {\scriptsize(\asrworse{0.0700})}
& 0.2167 {\scriptsize(\asrworse{0.0834})} \\

Weight noise
& \(\sigma=2\)
& 20.7501 {\scriptsize(\pplworse{12.4678})}
& 0.5993 {\scriptsize(\evalworse{0.0784})}
& 0.1667 {\scriptsize(\asrworse{0.0600})}
& 0.2233 {\scriptsize(\asrworse{0.1400})}
& 0.2833 {\scriptsize(\asrworse{0.1500})} \\

\midrule
\midrule
Activation noise
& \(\alpha=0.05\)
& 9.5776 {\scriptsize(\pplworse{1.2953})}
& 0.6496 {\scriptsize(\evalworse{0.0281})}
& 0.1417 {\scriptsize(\asrworse{0.0350})}
& 0.1200 {\scriptsize(\asrworse{0.0367})}
& 0.1650 {\scriptsize(\asrworse{0.0317})} \\

Activation noise
& \(\alpha=0.08\)
& 12.3162 {\scriptsize(\pplworse{4.0339})}
& 0.6058 {\scriptsize(\evalworse{0.0719})}
& 0.2033 {\scriptsize(\asrworse{0.0966})}
& 0.1933 {\scriptsize(\asrworse{0.1100})}
& 0.2367 {\scriptsize(\asrworse{0.1034})} \\

Activation noise
& \(\alpha=0.10\)
& 15.9009 {\scriptsize(\pplworse{7.6186})}
& 0.5687 {\scriptsize(\evalworse{0.1090})}
& 0.2567 {\scriptsize(\asrworse{0.1500})}
& 0.3733 {\scriptsize(\asrworse{0.2900})}
& 0.3150 {\scriptsize(\asrworse{0.1817})} \\

Activation noise
& \(\alpha=0.12\)
& 22.3274 {\scriptsize(\pplworse{14.0451})}
& 0.5181 {\scriptsize(\evalworse{0.1596})}
& 0.2333 {\scriptsize(\asrworse{0.1266})}
& 0.6267 {\scriptsize(\asrworse{0.5434})}
& 0.5600 {\scriptsize(\asrworse{0.4267})} \\

\bottomrule
\end{tabular}
}
\vspace{-1em}
\end{table}

\section{Experiments}
\subsection{Experiment Setups}
\partitle{\textbf{Models and Datasets}} 
We conduct experiments on Llama-3-8B-Instruct~\cite{grattafiori2024llama} and Qwen2-7B-Instruct~\cite{qwen2offcial2025}, using them as the base safety-aligned models for fine-tuning and robustness refinement. We use CB-Safety~\cite{zou2024improving} for both safety alignment and robustness refinement.

\partitle{\textbf{Evaluation Metrics}}
We evaluate utility and safety using three types of metrics.
For utility, we report WikiText-2 perplexity (PPL)~\cite{merity2016pointer} and the average zero-shot accuracy over 11 \texttt{lm-evaluation-harness} tasks~\cite{eval-harness}.
For safety, we report ASR on the non-hash HarmBench split~\cite{mazeika2024harmbench} using three measures: HarmBench ASR with the official HarmBench-Llama-2-13B-cls classifier, LlamaGuard3 ASR with Llama Guard 3-8B, and AdvBench-style refusal-prefix ASR~\cite{zou2023universal}.
Lower PPL and ASR are better, while higher lm-eval accuracy is better.

\partitle{\textbf{Perturbation Settings}}
We consider three simple post-alignment perturbations.
\textbf{Quantization} reduces numerical precision during deployment, where W4A16 denotes 4-bit weight quantization with 16-bit activations, and W4A4 denotes 4-bit weights with 4-bit activations.
\textbf{Parameter noise} adds absolute Gaussian noise to all model parameters, controlled by the noise scale \(\sigma\).
\textbf{Activation noise} injects Gaussian noise into hidden states, with the noise magnitude scaled by the standard deviation of the corresponding activations and controlled by the relative scale \(\alpha\).

\subsection{Safety Alignment Remains Fragile under Simple Perturbations}

Table~\ref{tab:original_perturbation_asr} shows that Llama-3-8B-Instruct remains fragile under simple post-alignment perturbations. 
Both parameter-level perturbations, including quantization and weight noise, and activation-level noise consistently increase ASR. 
The degradation becomes more severe under stronger perturbations, such as W4A4 quantization or larger activation noise. 
These results suggest that safety alignment is effective in the standard setting but remains vulnerable to lightweight parameter- and activation-level perturbations. Additional results on Qwen2-7B-Instruct are reported in Appendix~\ref{app:exps}.

\subsection{FO Safety Alignment Improves Safety Ability, while ZO Refinement Enhances Robustness}
Table~\ref{tab:three_stage_vertical} shows that 100-step FO safety alignment substantially reduces ASR across all settings, mainly by strengthening the model's \emph{\textbf{overall safety alignment ability}}. 
In contrast, the subsequent layer-robustness-aware ZO refinement is designed as a targeted \emph{\textbf{robustness enhancement}} stage. 
Although it introduces only small changes to PPL and lm-eval, it further reduces ASR in most perturbed settings, especially under W4A4 quantization and activation noise. 
These results indicate that FO alignment is effective for acquiring safety behavior, while layer-robustness-aware ZO refinement further strengthens the robustness against post-alignment perturbations. 
Consistent results are also observed on Qwen2-7B-Instruct in Appendix \ref{app:exps}.

\newcommand{\hbetter}[1]{\textcolor{goodgreen}{$\downarrow$#1}}
\newcommand{\hworse}[1]{\textcolor{badred}{$\uparrow$#1}}
\newcommand{\lbetter}[1]{\textcolor{goodgreen}{$\uparrow$#1}}
\newcommand{\lworth}[1]{\textcolor{badred}{$\downarrow$#1}}
\newcommand{\same}{\textcolor{neutralgray}{$\rightarrow$0.0000}}

\begin{table*}[t]
\centering
\small
\setlength{\tabcolsep}{4.4pt}
\renewcommand{\arraystretch}{1.3}
\caption{
Robustness and utility on Llama-3-8B-Instruct.
We report the original model, 100-step FO safety alignment, and subsequent layer-robustness-aware ZO refinement.
For Stage I, arrows indicate changes relative to the original model; for Stage II, arrows indicate changes relative to the 100-step FO model.
Lower PPL and ASR are better, while higher lm-eval is better.
}
\label{tab:three_stage_vertical}
\resizebox{\linewidth}{!}{
\begin{tabular}{llccccc}
\toprule
\textbf{Perturbation}
& \textbf{Level}
& \textbf{PPL} $\downarrow$
& \textbf{lm-eval} $\uparrow$
& \textbf{HarmBench ASR} $\downarrow$
& \textbf{LlamaGuard3 ASR} $\downarrow$
& \textbf{AdvBench ASR} $\downarrow$ \\
\midrule

\multicolumn{7}{c}{\textbf{Stage I: 100-step FO Safety Alignment (change vs. Original)}} \\
\midrule
-- & --
& 9.4069 {\scriptsize(\hworse{1.1246})}
& 0.6653 {\scriptsize(\lworth{0.0124})}
& 0.0667 {\scriptsize(\hbetter{0.0400})}
& 0.0567 {\scriptsize(\hbetter{0.0266})}
& 0.0733 {\scriptsize(\hbetter{0.0600})} \\
Quantization & W4A16
& 10.3574 {\scriptsize(\hworse{1.3237})}
& 0.6272 {\scriptsize(\lworth{0.0260})}
& 0.0900 {\scriptsize(\hbetter{0.1200})}
& 0.0833 {\scriptsize(\hbetter{0.1000})}
& 0.1367 {\scriptsize(\hbetter{0.0916})} \\
Quantization & W4A4
& 15.0531 {\scriptsize(\hworse{1.1890})}
& 0.5445 {\scriptsize(\lworth{0.0158})}
& 0.1467 {\scriptsize(\hbetter{0.1100})}
& 0.2500 {\scriptsize(\hbetter{0.1200})}
& 0.3167 {\scriptsize(\hbetter{0.0966})} \\

Weight Noise & $\sigma=2$
& 21.7101 {\scriptsize(\hworse{0.9600})}
& 0.6023 {\scriptsize(\lbetter{0.0030})}
& 0.1200 {\scriptsize(\hbetter{0.0467})}
& 0.1067 {\scriptsize(\hbetter{0.1166})}
& 0.1467 {\scriptsize(\hbetter{0.1366})} \\

Activation Noise & $\alpha=0.08$
& 13.3357 {\scriptsize(\hworse{1.0195})}
& 0.5991 {\scriptsize(\lworth{0.0067})}
& 0.1533 {\scriptsize(\hbetter{0.0500})}
& 0.1422 {\scriptsize(\hbetter{0.0511})}
& 0.1667 {\scriptsize(\hbetter{0.0700})} \\
Activation Noise & $\alpha=0.10$
& 16.7568 {\scriptsize(\hworse{0.8559})}
& 0.5664 {\scriptsize(\lworth{0.0023})}
& 0.2111 {\scriptsize(\hbetter{0.0456})}
& 0.2600 {\scriptsize(\hbetter{0.1133})}
& 0.2644 {\scriptsize(\hbetter{0.0506})} \\

\midrule
\multicolumn{7}{c}{\textbf{Stage II: Layer-robustness-aware ZO Refinement (change vs. Stage I)}} \\
\midrule
-- & --
& 9.4229 {\scriptsize(\hworse{0.0160})}
& 0.6655 {\scriptsize(\lbetter{0.0002})}
& 0.0700 {\scriptsize(\hworse{0.0033})}
& 0.0567 {\scriptsize(\same)}
& 0.0767 {\scriptsize(\hworse{0.0034})} \\
Quantization & W4A16
& 10.4091 {\scriptsize(\hworse{0.0517})}
& 0.6220 {\scriptsize(\lworth{0.0052})}
& 0.0900 {\scriptsize(\same)}
& 0.0733 {\scriptsize(\hbetter{0.0100})}
& 0.1000 {\scriptsize(\hbetter{0.0367})} \\
Quantization & W4A4
& 14.4467 {\scriptsize(\hbetter{0.6064})}
& 0.5440 {\scriptsize(\lworth{0.0005})}
& 0.1233 {\scriptsize(\hbetter{0.0234})}
& 0.1767 {\scriptsize(\hbetter{0.0733})}
& 0.2733 {\scriptsize(\hbetter{0.0434})} \\

Weight Noise & $\sigma=2$
& 21.5337 {\scriptsize(\hbetter{0.1764})}
& 0.6137 {\scriptsize(\lbetter{0.0114})}
& 0.1067 {\scriptsize(\hbetter{0.0133})}
& 0.0967 {\scriptsize(\hbetter{0.0100})}
& 0.1556 {\scriptsize(\hworse{0.0089})} \\

Activation Noise & $\alpha=0.08$
& 13.3717 {\scriptsize(\hworse{0.0360})}
& 0.6039 {\scriptsize(\lbetter{0.0048})}
& 0.1367 {\scriptsize(\hbetter{0.0166})}
& 0.1244 {\scriptsize(\hbetter{0.0178})}
& 0.1567 {\scriptsize(\hbetter{0.0100})} \\
Activation Noise & $\alpha=0.10$
& 16.8278 {\scriptsize(\hworse{0.0710})}
& 0.5683 {\scriptsize(\lbetter{0.0019})}
& 0.1989 {\scriptsize(\hbetter{0.0122})}
& 0.2367 {\scriptsize(\hbetter{0.0233})}
& 0.2333 {\scriptsize(\hbetter{0.0311})} \\

\bottomrule
\end{tabular}
}
\vspace{-1em}
\end{table*}

\subsection{Efficiency Analysis: Training Time and Storage Overhead}

\begin{wraptable}{r}{0.48\textwidth}
\vspace{-2pt}
\centering
\small
\setlength{\tabcolsep}{5pt}
\renewcommand{\arraystretch}{1.12}
\captionsetup{skip=2pt}
\caption{Runtime and GPU memory comparison.}
\label{tab:runtime_memory_comparison}
\begin{tabular}{lcc}
\toprule
\textbf{Metric} & \textbf{FO} & \textbf{ZO} \\
\midrule
Runtime 
& 1146.8 s 
& \cellcolor{teal!10}152.11 s \textbf{(0.13\(\times\))} \\
Peak memory 
& 82024 MiB 
& \cellcolor{teal!10}26421 MiB \textbf{(0.32\(\times\))} \\
\bottomrule
\end{tabular}
\vspace{-8pt}
\end{wraptable}

Table~\ref{tab:runtime_memory_comparison} compares the cost of the FO alignment stage and the subsequent ZO robustness refinement stage.
FO training is used to establish safety alignment, while ZO refinement is applied afterward to improve robustness.
The ZO stage adds only \(152.11\) seconds on top of the \(1146.8\)-second FO alignment stage, increasing the total training time by about \(13.3\%\).
It also uses only \(0.32\times\) the peak GPU memory of FO training.
These results indicate that ZO refinement provides a lightweight post-alignment robustness enhancement with modest additional cost.

\subsection{Ablation Study}
\partitle{Zeroth-Order Refinement \emph{Alone} Improves Robustness}

Table~\ref{tab:sft_zo_benchmark_asr_compact} shows that replacing late-stage FO updates with ZO refinement already improves safety robustness, even without robustness-critical layer selection.
Across most perturbation settings, FO+ZO achieves lower ASR than the FO-aligned model on HarmBench, LlamaGuard3, and AdvBench.
The gain is most evident under W4A16 quantization, where ASR drops from \(0.090/0.083/0.137\) to \(0.067/0.057/0.083\) across the three benchmarks.
Under activation noise, FO+ZO also consistently reduces ASR at both \(\alpha=0.05\) and \(\alpha=0.08\).
Overall, these results show that ZO refinement alone can provide post-alignment robustness gains before introducing robustness-critical layer selection.

\newcommand{\lworse}[1]{\textcolor{badred}{\scriptsize $\uparrow$#1}}

\begin{table}[t]
\centering
\small
\setlength{\tabcolsep}{4.2pt}
\renewcommand{\arraystretch}{1.18}
\caption{
Effect of ZO robustness refinement on safety benchmarks under different perturbations.
Each entry reports the ASR after FO+ZO refinement, with the change relative to the FO-aligned model shown in parentheses.
Lower ASR indicates stronger safety robustness.
}
\label{tab:sft_zo_benchmark_asr_compact}
\resizebox{\linewidth}{!}{
\begin{tabular}{lcccccc}
\toprule
\textbf{Benchmark}
& \multicolumn{2}{c}{\textbf{Quantization}}
& \multicolumn{2}{c}{\textbf{Weight Noise}}
& \multicolumn{2}{c}{\textbf{Activation Noise}} \\
\cmidrule(lr){2-3}
\cmidrule(lr){4-5}
\cmidrule(lr){6-7}
& \textbf{W4A16}
& \textbf{W4A4}
& \(\boldsymbol{\sigma=1}\)
& \(\boldsymbol{\sigma=2}\)
& \(\boldsymbol{\alpha=0.05}\)
& \(\boldsymbol{\alpha=0.08}\) \\
\midrule

HarmBench ASR $\downarrow$
& 0.0667 {\scriptsize(\hbetter{0.0233})}
& 0.1400 {\scriptsize(\hbetter{0.0067})}
& 0.0756 {\scriptsize(\same)}
& 0.1200 {\scriptsize(\same)}
& 0.0789 {\scriptsize(\hbetter{0.0100})}
& 0.1444 {\scriptsize(\hbetter{0.0089})} \\

LlamaGuard3 ASR $\downarrow$
& 0.0567 {\scriptsize(\hbetter{0.0266})}
& 0.1967 {\scriptsize(\hbetter{0.0533})}
& 0.0611 {\scriptsize(\hbetter{0.0011})}
& 0.1067 {\scriptsize(\same)}
& 0.0656 {\scriptsize(\hbetter{0.0088})}
& 0.1322 {\scriptsize(\hbetter{0.0100})} \\

AdvBench ASR $\downarrow$
& 0.0833 {\scriptsize(\hbetter{0.0534})}
& 0.3333 {\scriptsize(\hworse{0.0166})}
& 0.0867 {\scriptsize(\hbetter{0.0011})}
& 0.1467 {\scriptsize(\same)}
& 0.0822 {\scriptsize(\hbetter{0.0067})}
& 0.1567 {\scriptsize(\hbetter{0.0100})} \\

\bottomrule
\end{tabular}
}
\end{table}

\begin{figure}
    \centering
    \includegraphics[width=0.8\linewidth]{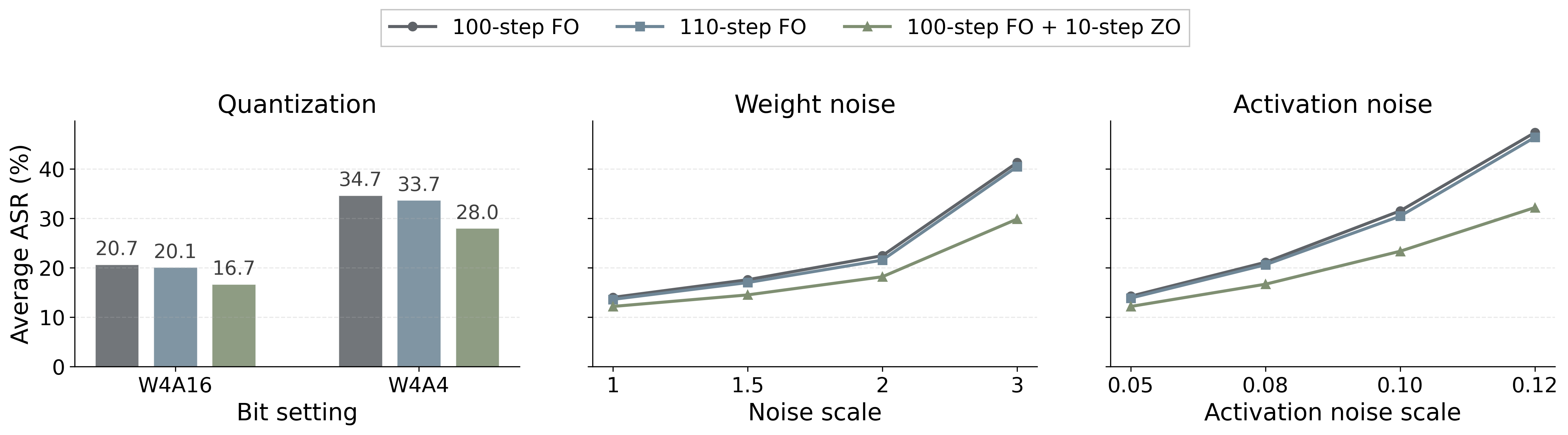}
    \caption{
Effect of replacing late-stage FO updates with ZO refinement. 
We compare 100-step FO, 110-step FO, and 100-step FO followed by 10-step ZO, reporting the average ASR over HarmBench, LlamaGuard3, and AdvBench.
Replacing the last 10 FO steps with ZO consistently lowers ASR across all perturbations while reducing memory cost by avoiding backpropagation.
}
\vspace{-0.5em}
\label{fig:ablation_fo_zo}
\end{figure}

\begin{wrapfigure}[14]{r}{0.5\linewidth}
    \centering
    \vspace{-1.1em}
    \includegraphics[width=0.9\linewidth]{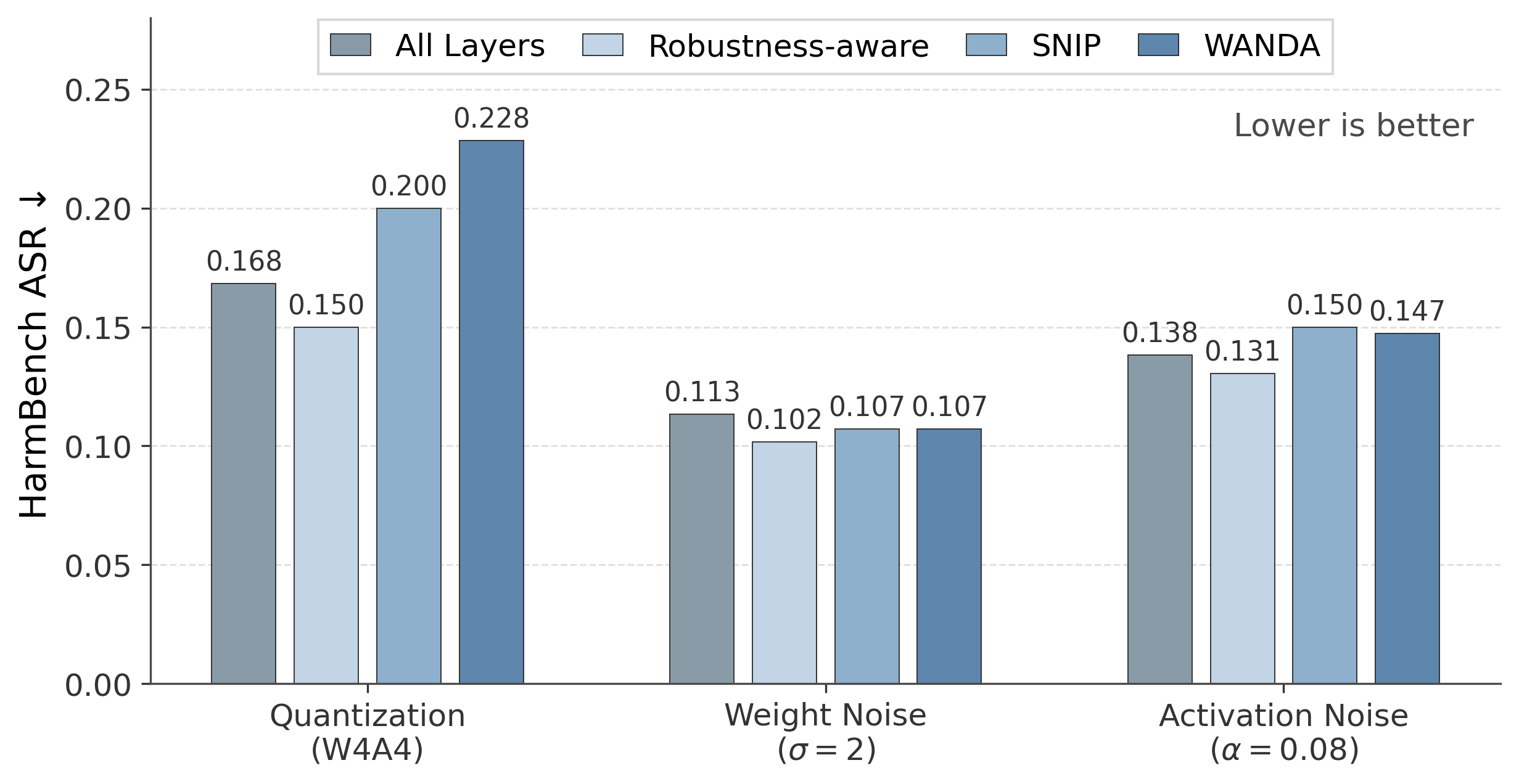}
    \caption{
    HarmBench ASR after ZO refinement with different layer-selection strategies. Lower ASR is better. Robustness-aware selection is consistently best or comparable across perturbations.
    }
    \label{fig:selection_comparison}
\end{wrapfigure}

\partitle{Replacing Late-Stage FO with ZO Improves Robustness and Reduces Memory}
We further study whether the robustness gain comes from additional training steps or from the use of ZO refinement itself. To this end, we compare three variants: 100-step FO alignment, continued 110-step FO alignment, and 100-step FO alignment followed by 10-step ZO refinement. The comparison between 110-step FO and 100-step FO + 10-step ZO is particularly important, since both use the same total number of optimization steps, but differ in the optimizer used during the final stage.
As shown in Figure~\ref{fig:ablation_fo_zo}, replacing the last few FO steps with ZO refinement consistently improves robustness compared with simply continuing FO training. Under quantization, weight noise, and activation noise, the FO+ZO variant achieves lower average ASR across the three safety benchmarks. Meanwhile, ZO refinement avoids backpropagation-based gradient computation and therefore reduces GPU memory consumption during the refinement stage. These results suggest that late-stage ZO refinement provides a better robustness--efficiency trade-off than continued FO fine-tuning.

\partitle{Layer-Wise Robustness Selection Enables Stronger ZO Refinement} Figure~\ref{fig:selection_comparison} compares different parameter-selection strategies for ZO robustness refinement. Robustness-aware selection achieves the lowest HarmBench ASR across all perturbation settings, showing that it can effectively identify layers that are most responsible for safety degradation under perturbations.
In contrast, SNIP and WANDA perform well under weight noise but lead to noticeably higher ASR under perturbation.
This suggests that the proposed robustness-aware criterion provides a more reliable signal for targeted ZO refinement than the pruning-based importance scores.

\section{Conclusion}

In this work, we study the robustness of safety-aligned LLMs under simple post-alignment perturbations, including quantization, weight noise, and activation noise.
We theoretically and empirically show that standard safety alignment can remain fragile, and propose a perturbation-guided FO-to-ZO refinement framework where FO establishes safety alignment and ZO further improves robustness.
By identifying robustness-critical layers, we restrict ZO updates to these layers for more efficient refinement.
Future work includes extending our layer-level selection to finer-grained robustness attribution at the module, neuron, or parameter level.


\bibliographystyle{IEEEtran} 
\bibliography{bib}

@article{shahani2025noise,
  title={Noise Injection Systemically Degrades Large Language Model Safety Guardrails},
  author={Shahani, Prithviraj Singh and Miandoab, Kaveh Eskandari and Scheutz, Matthias},
  journal={arXiv preprint arXiv:2505.13500},
  year={2025}
}

@article{spall1992multivariate,
  title={Multivariate stochastic approximation using a simultaneous perturbation gradient approximation},
  author={Spall, James C},
  journal={IEEE transactions on automatic control},
  volume={37},
  number={3},
  pages={332--341},
  year={1992},
  publisher={IEEE}
}

@article{liu2026differentially,
  title={Differentially Private Zeroth-Order Methods for Scalable Large Language Model Fine-tuning},
  author={Liu, Zhihao and Lou, Jian and Bao, Wenjie and Hu, Yuke and Wang, Wenli and Qin, Zhan and Ren, Kui},
  journal={IEEE Transactions on Information Forensics and Security},
  year={2026},
  publisher={IEEE}
}

@article{ouyang2022training,
  title={Training language models to follow instructions with human feedback},
  author={Ouyang, Long and others},
  journal={arXiv preprint arXiv:2203.02155},
  year={2022}
}

@article{rafailov2023direct,
  title={Direct Preference Optimization: Your Language Model is Secretly a Reward Model},
  author={Rafailov, Rafael and others},
  journal={arXiv preprint arXiv:2305.18290},
  year={2023}
}

@article{dettmers2022llm,
  title={LLM.int8(): 8-bit Matrix Multiplication for Transformers at Scale},
  author={Dettmers, Tim and others},
  journal={arXiv preprint arXiv:2208.07339},
  year={2022}
}

@article{nesterov2017random,
  title={Random gradient-free minimization of convex functions},
  author={Nesterov, Yurii and Spokoiny, Vladimir},
  journal={Foundations of Computational Mathematics},
  year={2017}
}

@article{liu2020primer,
  title={A primer on zeroth-order optimization in signal processing and machine learning: Principals, recent advances, and applications},
  author={Liu, Sijia and Chen, Pin-Yu and Kailkhura, Bhavya and Zhang, Gaoyuan and Hero III, Alfred O and Varshney, Pramod K},
  journal={IEEE Signal Processing Magazine},
  volume={37},
  number={5},
  pages={43--54},
  year={2020},
  publisher={IEEE}
}

@misc{qwen2offcial2025,
  title={Qwen2 Technical Report},
  author={Qwen Team, Alibaba Cloud},
  howpublished={\url{https://qwen2.org/paper/}},
  year={2025}
}

@article{touvron2023llama,
  title={Llama 2: Open Foundation and Fine-Tuned Chat Models},
  author={Touvron, Hugo and Martin, Louis and Stone, Kevin and Albert, Peter and Almahairi, Amjad and Babaei, Yasmine and Bashlykov, Nikolay and Batra, Soumya and Bhargava, Prajjwal and Bhosale, Shruti and others},
  journal={arXiv preprint arXiv:2307.09288},
  year={2023}
}

@article{ethayarajh2024kto,
  title={Kto: Model alignment as prospect theoretic optimization},
  author={Ethayarajh, Kawin and Xu, Winnie and Muennighoff, Niklas and Jurafsky, Dan and Kiela, Douwe},
  journal={arXiv preprint arXiv:2402.01306},
  year={2024}
}

@article{guan2025deliberative,
  title={Deliberative Alignment: Reasoning Enables Safer Language Models},
  author={Guan, Melody Y and Joglekar, Manas and Wallace, Eric and Jain, Saachi and Barak, Boaz and Helyar, Alec and Dias, Rachel and Vallone, Andrea and Ren, Hongyu and Wei, Jason and others},
  journal={SuperIntelligence-Robotics-Safety \& Alignment},
  volume={2},
  number={3},
  year={2025}
}

@article{weidinger2021ethical,
  title={Ethical and social risks of harm from language models},
  author={Weidinger, Laura and Mellor, John and Rauh, Maribeth and Griffin, Conor and Uesato, Jonathan and Huang, Po-Sen and Cheng, Myra and Glaese, Mia and Balle, Borja and Kasirzadeh, Atoosa and others},
  journal={arXiv preprint arXiv:2112.04359},
  year={2021}
}

@inproceedings{gehman2020realtoxicityprompts,
  title={Realtoxicityprompts: Evaluating neural toxic degeneration in language models},
  author={Gehman, Samuel and Gururangan, Suchin and Sap, Maarten and Choi, Yejin and Smith, Noah A},
  booktitle={Findings of the association for computational linguistics: EMNLP 2020},
  pages={3356--3369},
  year={2020}
}

@article{liu2023jailbreaking,
  title={Jailbreaking chatgpt via prompt engineering: An empirical study},
  author={Liu, Yi and Deng, Gelei and Xu, Zhengzi and Li, Yuekang and Zheng, Yaowen and Zhang, Ying and Zhao, Lida and Zhang, Tianwei and Wang, Kailong and Liu, Yang},
  journal={arXiv preprint arXiv:2305.13860},
  year={2023}
}

@inproceedings{tice2024sandbag,
  title={Sandbag Detection through Model Impairment},
  author={Tice, Cameron and Kreer, Philipp Alexander and Helm-Burger, Nathan and Shahani, Prithviraj Singh and Ryzhenkov, Fedor and van der Weij, Teun and Hofst{\"a}tter, Felix and Haimes, Jacob},
  booktitle={Workshop on Socially Responsible Language Modelling Research},
  year={2024}
}

@article{clymer2024poser,
  title={Poser: Unmasking alignment faking llms by manipulating their internals},
  author={Clymer, Joshua and Juang, Caden and Field, Severin},
  journal={arXiv preprint arXiv:2405.05466},
  year={2024}
}

@article{liu2024robustifying,
  title={Robustifying safety-aligned large language models through clean data curation},
  author={Liu, Xiaoqun and Liang, Jiacheng and Ye, Muchao and Xi, Zhaohan},
  journal={arXiv preprint arXiv:2405.19358},
  year={2024}
}

@inproceedings{shen2025seal,
  title={SEAL: SAFETY-ENHANCED ALIGNED LLM FINE-TUNING VIA BILEVEL DATA SELECTION},
  author={Shen, Han and Chen, Pin-Yu and Das, Payel and Chen, Tianyi},
  booktitle={International Conference on Learning Representations},
  year={2025}
}

@inproceedings{qi2024safety,
  title={Safety Alignment Should be Made More Than Just a Few Tokens Deep},
  author={Qi, Xiangyu and Panda, Ashwinee and Lyu, Kaifeng and Ma, Xiao and Roy, Subhrajit and Beirami, Ahmad and Mittal, Prateek and Henderson, Peter},
  booktitle={The Thirteenth International Conference on Learning Representations},
  year={2025}
}

@inproceedings{zhao2025improving,
  title={Improving LLM Safety Alignment with Dual-Objective Optimization},
  author={Zhao, Xuandong and Cai, Will and Shi, Tianneng and Huang, David and Lin, Licong and Mei, Song and Song, Dawn},
  booktitle={Forty-second International Conference on Machine Learning},
  year={2025}
}

@article{huang2024vaccine,
  title={Vaccine: Perturbation-aware alignment for large language models against harmful fine-tuning attack},
  author={Huang, Tiansheng and Hu, Sihao and Liu, Ling},
  journal={Advances in Neural Information Processing Systems},
  volume={37},
  pages={74058--74088},
  year={2024}
}

@article{rosati2024representation,
  title={Representation noising: A defence mechanism against harmful finetuning},
  author={Rosati, Domenic and Wehner, Jan and Williams, Kai and Bartoszcze, {\L}ukasz and Atanasov, David and Gonzales, Robie and Majumdar, Subhabrata and Maple, Carsten and Sajjad, Hassan and Rudzicz, Frank},
  journal={Advances in Neural Information Processing Systems},
  volume={37},
  pages={12636--12676},
  year={2024}
}

@inproceedings{li2024safety,
  title={Safety Layers in Aligned Large Language Models: The Key to LLM Security},
  author={Li, Shen and Yao, Liuyi and Zhang, Lan and Li, Yaliang},
  booktitle={The Thirteenth International Conference on Learning Representations},
  year={2025}
}

@article{wei2023jailbroken,
  title={Jailbroken: How does llm safety training fail?},
  author={Wei, Alexander and Haghtalab, Nika and Steinhardt, Jacob},
  journal={Advances in neural information processing systems},
  volume={36},
  pages={80079--80110},
  year={2023}
}

@article{zou2023universal,
  title={Universal and transferable adversarial attacks on aligned language models},
  author={Zou, Andy and Wang, Zifan and Carlini, Nicholas and Nasr, Milad and Kolter, J Zico and Fredrikson, Matt},
  journal={arXiv preprint arXiv:2307.15043},
  year={2023}
}

@inproceedings{shen2024anything,
  title={" do anything now": Characterizing and evaluating in-the-wild jailbreak prompts on large language models},
  author={Shen, Xinyue and Chen, Zeyuan and Backes, Michael and Shen, Yun and Zhang, Yang},
  booktitle={Proceedings of the 2024 on ACM SIGSAC Conference on Computer and Communications Security},
  pages={1671--1685},
  year={2024}
}

@inproceedings{qi2024fine,
  title={Fine-tuning Aligned Language Models Compromises Safety, Even When Users Do Not Intend To!},
  author={Qi, Xiangyu and Zeng, Yi and Xie, Tinghao and Chen, Pin-Yu and Jia, Ruoxi and Mittal, Prateek and Henderson, Peter},
  booktitle={International Conference on Learning Representations},
  year={2024}
}

@article{yang2023shadow,
  title={Shadow alignment: The ease of subverting safely-aligned language models},
  author={Yang, Xianjun and Wang, Xiao and Zhang, Qi and Petzold, Linda and Wang, William Yang and Zhao, Xun and Lin, Dahua},
  journal={arXiv preprint arXiv:2310.02949},
  year={2023}
}

@article{hubinger2024sleeper,
  title={Sleeper agents: Training deceptive llms that persist through safety training},
  author={Hubinger, Evan and Denison, Carson and Mu, Jesse and Lambert, Mike and Tong, Meg and MacDiarmid, Monte and Lanham, Tamera and Ziegler, Daniel M and Maxwell, Tim and Cheng, Newton and others},
  journal={arXiv preprint arXiv:2401.05566},
  year={2024}
}

@article{arditi2024refusal,
  title={Refusal in language models is mediated by a single direction},
  author={Arditi, Andy and Obeso, Oscar and Syed, Aaquib and Paleka, Daniel and Panickssery, Nina and Gurnee, Wes and Nanda, Neel},
  journal={Advances in Neural Information Processing Systems},
  volume={37},
  pages={136037--136083},
  year={2024}
}

@inproceedings{zahran2025jailbreaking,
  title={On jailbreaking quantized language models through fault injection attacks},
  author={Zahran, Noureldin and Tahmasivand, Ahmad and Alouani, Ihsen and Khasawneh, Khaled and Fouda, Mohammed},
  booktitle={Proceedings of the Great Lakes Symposium on VLSI 2025},
  pages={554--561},
  year={2025}
}

@inproceedings{ticenoise,
  title={Noise Injection Reveals Hidden Capabilities of Sandbagging Language Models},
  author={Tice, Cameron and Kreer, Philipp Alexander and Helm-Burger, Nathan and Shahani, Prithviraj Singh and Ryzhenkov, Fedor and Roger, Fabien and Neo, Clement and Haimes, Jacob and Hofst{\"a}tter, Felix and van der Weij, Teun},
  booktitle={The Thirty-ninth Annual Conference on Neural Information Processing Systems},
  year={2025}
}

@inproceedings{galatolo2025visualising,
  title={Visualising Policy-Reward Interplay to Inform Zeroth-Order Preference Optimisation of Large Language Models},
  author={Galatolo, Alessio and Dai, Zhenbang and Winkle, Katie and Beloucif, Meriem},
  booktitle={Findings of the Association for Computational Linguistics: ACL 2025},
  pages={17446--17461},
  year={2025}
}

@article{malladi2023fine,
  title={Fine-tuning language models with just forward passes},
  author={Malladi, Sadhika and Gao, Tianyu and Nichani, Eshaan and Damian, Alex and Lee, Jason D and Chen, Danqi and Arora, Sanjeev},
  journal={Advances in Neural Information Processing Systems},
  volume={36},
  pages={53038--53075},
  year={2023}
}

@inproceedings{liu2024zeroth,
  title={Zeroth-order methods for constrained nonconvex nonsmooth stochastic optimization},
  author={Liu, Zhuanghua and Chen, Cheng and Luo, Luo and Low, Bryan Kian Hsiang},
  booktitle={Forty-first International Conference on Machine Learning},
  year={2024}
}

@inproceedings{karimi2016linear,
  title={Linear convergence of gradient and proximal-gradient methods under the polyak-{\l}ojasiewicz condition},
  author={Karimi, Hamed and Nutini, Julie and Schmidt, Mark},
  booktitle={Joint European conference on machine learning and knowledge discovery in databases},
  pages={795--811},
  year={2016},
  organization={Springer}
}

@inproceedings{mazeika2024harmbench,
  title={HarmBench: A Standardized Evaluation Framework for Automated Red Teaming and Robust Refusal},
  author={Mazeika, Mantas and Phan, Long and Yin, Xuwang and Zou, Andy and Wang, Zifan and Mu, Norman and Sakhaee, Elham and Li, Nathaniel and Basart, Steven and Li, Bo and others},
  booktitle={International Conference on Machine Learning},
  pages={35181--35224},
  year={2024},
  organization={PMLR}
}

@inproceedings{leesnip,
  title={SNIP: SINGLE-SHOT NETWORK PRUNING BASED ON CONNECTION SENSITIVITY},
  author={Lee, Namhoon and Ajanthan, Thalaiyasingam and Torr, Philip},
  booktitle={International Conference on Learning Representations},
  year={2019}
}

@inproceedings{sun2024simple,
  title={A SIMPLE AND EFFECTIVE PRUNING APPROACH FOR LARGE LANGUAGE MODELS},
  author={Sun, Mingjie and Liu, Zhuang and Bair, Anna and Kolter, J Zico},
  booktitle={12th International Conference on Learning Representations, ICLR 2024},
  year={2024}
}

@article{grattafiori2024llama,
  title={The llama 3 herd of models},
  author={Grattafiori, Aaron and Dubey, Abhimanyu and Jauhri, Abhinav and Pandey, Abhinav and Kadian, Abhishek and Al-Dahle, Ahmad and Letman, Aiesha and Mathur, Akhil and Schelten, Alan and Vaughan, Alex and others},
  journal={arXiv preprint arXiv:2407.21783},
  year={2024}
}

@article{zou2024improving,
  title={Improving alignment and robustness with circuit breakers},
  author={Zou, Andy and Phan, Long and Wang, Justin and Duenas, Derek and Lin, Maxwell and Andriushchenko, Maksym and Wang, Rowan and Kolter, Zico and Fredrikson, Matt and Hendrycks, Dan},
  journal={Advances in Neural Information Processing Systems},
  volume={37},
  pages={83345--83373},
  year={2024}
}

@article{merity2016pointer,
  title={Pointer sentinel mixture models},
  author={Merity, Stephen and Xiong, Caiming and Bradbury, James and Socher, Richard},
  journal={arXiv preprint arXiv:1609.07843},
  year={2016}
}

@misc{eval-harness,
  author       = {Gao, Leo and Tow, Jonathan and Abbasi, Baber and Biderman, Stella and Black, Sid and DiPofi, Anthony and Foster, Charles and Golding, Laurence and Hsu, Jeffrey and Le Noac'h, Alain and Li, Haonan and McDonell, Kyle and Muennighoff, Niklas and Ociepa, Chris and Phang, Jason and Reynolds, Laria and Schoelkopf, Hailey and Skowron, Aviya and Sutawika, Lintang and Tang, Eric and Thite, Anish and Wang, Ben and Wang, Kevin and Zou, Andy},
  title        = {The Language Model Evaluation Harness},
  month        = 07,
  year         = 2024,
  publisher    = {Zenodo},
  version      = {v0.4.3},
  doi          = {10.5281/zenodo.12608602},
  url          = {https://zenodo.org/records/12608602}
}

@article{bai2022constitutional,
  title={Constitutional ai: Harmlessness from ai feedback},
  author={Bai, Yuntao and Kadavath, Saurav and Kundu, Sandipan and Askell, Amanda and Kernion, Jackson and Jones, Andy and Chen, Anna and Goldie, Anna and Mirhoseini, Azalia and McKinnon, Cameron and others},
  journal={arXiv preprint arXiv:2212.08073},
  year={2022}
}

@inproceedings{dai2024safe,
  title={Safe RLHF: Safe Reinforcement Learning from Human Feedback},
  author={Dai, Josef and Pan, Xuehai and Sun, Ruiyang and Ji, Jiaming and Xu, Xinbo and Liu, Mickel and Wang, Yizhou and Yang, Yaodong},
  booktitle={The Twelfth International Conference on Learning Representations},
  year={2024}
}

@article{mou2025saro,
  title={Saro: Enhancing LLM safety through reasoning-based alignment},
  author={Mou, Yutao and Luo, Yuxiao and Zhang, Shikun and Ye, Wei},
  journal={arXiv preprint arXiv:2504.09420},
  year={2025}
}

@article{wang2025comprehensive,
  title={A comprehensive survey in llm (-agent) full stack safety: Data, training and deployment},
  author={Wang, Kun and Zhang, Guibin and Zhou, Zhenhong and Wu, Jiahao and Yu, Miao and Zhao, Shiqian and Yin, Chenlong and Fu, Jinhu and Yan, Yibo and Luo, Hanjun and others},
  journal={arXiv preprint arXiv:2504.15585},
  year={2025}
}

@inproceedings{ji2025pku,
  title={Pku-saferlhf: Towards multi-level safety alignment for llms with human preference},
  author={Ji, Jiaming and Hong, Donghai and Zhang, Borong and Chen, Boyuan and Dai, Josef and Zheng, Boren and Qiu, Tianyi Alex and Zhou, Jiayi and Wang, Kaile and Li, Boxun and others},
  booktitle={Proceedings of the 63rd Annual Meeting of the Association for Computational Linguistics (Volume 1: Long Papers)},
  pages={31983--32016},
  year={2025}
}

@inproceedings{zhang2025towards,
  title={Towards memory-efficient and sustainable machine unlearning on edge using zeroth-order optimizer},
  author={Zhang, Ci and Yang, Chence and Tan, Qitao and Liu, Jun and Li, Ao and Wang, Yanzhi and Lu, Jin and Wang, Jinhui and Yuan, Geng},
  booktitle={Proceedings of the Great Lakes Symposium on VLSI 2025},
  pages={227--232},
  year={2025}
}

@article{zhang2024revisiting,
  title={Revisiting zeroth-order optimization for memory-efficient llm fine-tuning: A benchmark},
  author={Zhang, Yihua and Li, Pingzhi and Hong, Junyuan and Li, Jiaxiang and Zhang, Yimeng and Zheng, Wenqing and Chen, Pin-Yu and Lee, Jason D and Yin, Wotao and Hong, Mingyi and others},
  journal={arXiv preprint arXiv:2402.11592},
  year={2024}
}

@article{lang2026downgrade,
  title={Downgrade to Upgrade: Optimizer Simplification Enhances Robustness in LLM Unlearning},
  author={Lang, Yicheng and Zhang, Yihua and Fan, Chongyu and Wang, Changsheng and Jia, Jinghan and Liu, Sijia},
  journal={14th International Conference on Learning Representations, ICLR 2026},
  year={2026}
}

@inproceedings{zhang2025zeroth,
  title={Zeroth-Order Optimization Finds Flat Minima},
  author={Zhang, Liang and Li, Bingcong and Thekumparampil, Kiran Koshy and Oh, Sewoong and Muehlebach, Michael and He, Niao},
  booktitle={The Thirty-ninth Annual Conference on Neural Information Processing Systems},
  year={2025}
}

@article{walter2025flatness,
  title={When flatness does (not) guarantee adversarial robustness},
  author={Walter, Nils Philipp and Adilova, Linara and Vreeken, Jilles and Kamp, Michael},
  journal={arXiv preprint arXiv:2510.14231},
  year={2025}
}

\appendix

\appendix
\section{Hyperparameters and Hardware Configuration}
\label{app:hyperparams}

\subsection{Training Hyperparameters}
The main hyperparameters used for safety alignment and zeroth-order robustness refinement are summarized in Table~\ref{tab:hyperparams}. We use the same data preprocessing and evaluation protocol across all compared methods unless otherwise specified. All stages use a per-device batch size of 1 and gradient accumulation of 8.

\begin{table}[h]
\centering
\small
\caption{Hyperparameters for FO safety alignment and ZO robustness refinement.}
\begin{tabular}{l l l}
\toprule
Stage & Hyperparameter & Value \\
\midrule
\multicolumn{3}{l}{\textbf{General setup}} \\
\midrule
General & Base models & Llama-3-8B-Instruct, Qwen2-7B \\
General & Training dataset & CB-Safety \\
General & Training sample size & 800 \\
General & Maximum sequence length & 1536 \\
General & Per-device batch size & 1 \\
General & Gradient accumulation steps & 8 \\
\midrule
\multicolumn{3}{l}{\textbf{Stage I: FO safety alignment}} \\
\midrule
FO alignment & Optimizer & SGD \\
FO alignment & Training steps & 100 \\
FO alignment & Learning rate $\eta_{\mathrm{fo}}$ &
$\{5{\times}10^{-3},\,8{\times}10^{-3},\,1{\times}10^{-2},\,1.2{\times}10^{-2}\}$ \\
FO alignment & Momentum &
$\{0,\,0.5\}$ \\
FO alignment & Scheduler &
$\{\texttt{constant\_with\_warmup},\,\texttt{cosine}\}$ \\
FO alignment & Warmup ratio &
$\{0.05,\,0.08\}$ \\
FO alignment & Weight decay & 0 \\
\midrule
\multicolumn{3}{l}{\textbf{Stage II: ZO refinement}} \\
\midrule
ZO refinement & Training steps & 10 \\
ZO refinement & Learning rate $\eta_{\mathrm{zo}}$ &
$\{5{\times}10^{-3},\,8{\times}10^{-3},\,1{\times}10^{-2},\,1.2{\times}10^{-2}\}$ \\
ZO refinement & ZO scale parameter $\beta$ &
$\{1{\times}10^{-4},\,5{\times}10^{-4},\,8{\times}10^{-4},\,1{\times}10^{-3},\,3{\times}10^{-3}\}$ \\
ZO refinement & Samples per update & 8 \\
ZO refinement & ZO momentum & 0 \\
ZO refinement & Scheduler & \texttt{cosine} \\
ZO refinement & Warmup ratio & 0 \\
ZO refinement & Weight decay & $1{\times}10^{-4}$ \\
\midrule
\multicolumn{3}{l}{\textbf{Layer-robustness-aware ZO refinement}} \\
\midrule
Focused ZO & Learning rate $\eta_{\mathrm{zo}}$ &
$\{1{\times}10^{-3},\,4{\times}10^{-3},\,1{\times}10^{-2},\,1.5{\times}10^{-2}\}$ \\
Focused ZO & ZO scale parameter $\beta$ &
$\{8{\times}10^{-4},\,1{\times}10^{-3},\,1.2{\times}10^{-3}\}$ \\
Focused ZO & Checkpoint dtype & bfloat16 \\
\bottomrule
\end{tabular}

\label{tab:hyperparams}
\end{table}

\subsection{Hardware Configuration}
The Llama-3-8B-Instruct experiments were run on a local RTX 4090 server. The machine contains 8 NVIDIA GeForce RTX 4090 GPUs with 24GB memory each. The Qwen2-7B-Instruct experiments were run on RTX 4090 GPUs with 48GB memory. Training uses multi-GPU execution, while quantization and safety evaluation use fewer GPUs when possible.

\section{Theoretical Proofs}
\label{app:zo-proofs}

\subsection{Proof of ZO Convergence}
\label{app:proof-zo-convergence}

We first prove the convergence guarantee for ZO optimization on the smoothed objective \(f_\beta\).
Recall the two-sided ZO estimator
\[
    \widehat g(\theta;v)
    :=
    \frac{d}{2\beta}
    \bigl(
        f(\theta+\beta v)-f(\theta-\beta v)
    \bigr)v .
\]
We use the following standard variance bound.

\begin{lemma}[Variance bound for the ZO estimator~\cite{liu2026differentially}]
\label{lem:zo-variance}
For the estimator \(\widehat g(\theta;v)\), we have
\[
    \mathbb{E}_{v}[\widehat g(\theta;v)]
    =
    \nabla f_\beta(\theta),
\]
and
\[
    \mathbb{E}_{v}
    \bigl[
        \|\widehat g(\theta;v)-\nabla f_\beta(\theta)\|^2
    \bigr]
    \le
    64d\beta^2L^4 .
\]
\end{lemma}

\begin{proof}[Proof of Theorem~\ref{thm:zo-global-convergence}]
Let
\[
    g_t := \nabla f_\beta(\theta_t),
    \qquad
    \widehat g_t := \widehat g(\theta_t;v_t).
\]
The ZO update is
\[
    \theta_{t+1}
    =
    \theta_t-\eta \widehat g_t .
\]
By Taylor's theorem, there exists a point \(\widetilde\theta_t\) on the line segment between \(\theta_t\) and \(\theta_{t+1}\) such that
\[
\begin{aligned}
    f_\beta(\theta_{t+1})
    \le
    f_\beta(\theta_t)
    +
    g_t^\top(\theta_{t+1}-\theta_t)  
    +
    \frac12
    (\theta_{t+1}-\theta_t)^\top
    \nabla^2 f_\beta(\widetilde\theta_t)
    (\theta_{t+1}-\theta_t).
\end{aligned}
\]
Substituting \(\theta_{t+1}-\theta_t=-\eta\widehat g_t\), we obtain
\[
    f_\beta(\theta_{t+1})
    \le
    f_\beta(\theta_t)
    -
    \eta g_t^\top\widehat g_t
    +
    \frac{\eta^2}{2}
    \widehat g_t^\top
    \nabla^2 f_\beta(\widetilde\theta_t)
    \widehat g_t .
\]
By Assumption~\ref{ass:low-rank-curvature},
\[
    \nabla^2 f_\beta(\widetilde\theta_t)\preceq H_t .
\]
Therefore,
\[
    f_\beta(\theta_{t+1})
    \le
    f_\beta(\theta_t)
    -
    \eta g_t^\top\widehat g_t
    +
    \frac{\eta^2}{2}
    \widehat g_t^\top H_t\widehat g_t .
\]

Taking conditional expectation over the ZO randomness at step \(t\), denoted by \(\mathbb{E}_t[\cdot]\), and using Lemma~\ref{lem:zo-variance}, we have
\[
    \mathbb{E}_t[\widehat g_t]=g_t.
\]
Thus,
\[
    \mathbb{E}_t[f_\beta(\theta_{t+1})]
    \le
    f_\beta(\theta_t)
    -
    \eta\|g_t\|^2
    +
    \frac{\eta^2}{2}
    \mathbb{E}_t[
        \widehat g_t^\top H_t\widehat g_t
    ].
\]

It remains to bound the quadratic term.
Since \(H_t\succeq 0\),
\[
    \widehat g_t^\top H_t\widehat g_t
    \le
    \operatorname{tr}(H_t)\|\widehat g_t\|^2 .
\]
Moreover, by the effective-rank condition and \(\|H_t\|_{\op}\le \ell\),
\[
    \operatorname{tr}(H_t)
    =
    \erank(H_t)\|H_t\|_{\op}
    \le
    r\ell .
\]
Hence,
\[
    \mathbb{E}_t[
        \widehat g_t^\top H_t\widehat g_t
    ]
    \le
    \ell r\,
    \mathbb{E}_t[\|\widehat g_t\|^2].
\]

Using variance decomposition,
\[
\begin{aligned}
    \mathbb{E}_t[\|\widehat g_t\|^2]
    &=
    \|\mathbb{E}_t[\widehat g_t]\|^2
    +
    \mathbb{E}_t[
        \|\widehat g_t-\mathbb{E}_t[\widehat g_t]\|^2
    ]  \\
    &\le
    \|g_t\|^2
    +
    64d\beta^2L^4 .
\end{aligned}
\]
Substituting this bound yields
\[
\begin{aligned}
    \mathbb{E}_t[f_\beta(\theta_{t+1})]
    \le
    f_\beta(\theta_t)
    &-
    \eta\|g_t\|^2  \\
    &+
    \frac{\eta^2\ell r}{2}
    \bigl(
        \|g_t\|^2+64d\beta^2L^4
    \bigr).
\end{aligned}
\]
Equivalently,
\[
    \mathbb{E}_t[f_\beta(\theta_{t+1})]
    \le
    f_\beta(\theta_t)
    -
    \left(
        \eta-\frac{\eta^2\ell r}{2}
    \right)
    \|g_t\|^2
    +
    32\eta^2\ell r d\beta^2L^4 .
\]

If \(0<\eta\le 1/(\ell r)\), then
\[
    \eta-\frac{\eta^2\ell r}{2}
    \ge
    \frac{\eta}{2}.
\]
Therefore,
\[
    \mathbb{E}_t[f_\beta(\theta_{t+1})]
    \le
    f_\beta(\theta_t)
    -
    \frac{\eta}{2}\|g_t\|^2
    +
    32\eta^2\ell r d\beta^2L^4 .
\]

By the \(\mu\)-PL condition on \(f_\beta\),
\[
    \frac12\|g_t\|^2
    \ge
    \mu\bigl(
        f_\beta(\theta_t)-f_\beta^\star
    \bigr).
\]
Thus,
\[
    \mathbb{E}_t[
        f_\beta(\theta_{t+1})-f_\beta^\star
    ]
    \le
    (1-\mu\eta)
    \bigl(
        f_\beta(\theta_t)-f_\beta^\star
    \bigr)
    +
    32\eta^2\ell r d\beta^2L^4 .
\]
Taking total expectation and defining
\[
    \Delta_t
    :=
    \mathbb{E}[
        f_\beta(\theta_t)-f_\beta^\star
    ],
\]
we obtain the recursion
\[
    \Delta_{t+1}
    \le
    (1-\mu\eta)\Delta_t
    +
    32\eta^2\ell r d\beta^2L^4 .
\]
Unrolling gives
\[
\begin{aligned}
    \Delta_t
    &\le
    (1-\mu\eta)^t\Delta_0
    +
    32\eta^2\ell r d\beta^2L^4
    \sum_{i=0}^{t-1}(1-\mu\eta)^i  \\
    &\le
    (1-\mu\eta)^t\Delta_0
    +
    \frac{
        32\eta\ell r d\beta^2L^4
    }{\mu}.
\end{aligned}
\]
Using \(1-x\le e^{-x}\), we conclude that
\[
    \Delta_t
    \le
    e^{-\mu\eta t}\Delta_0
    +
    \frac{
        32\eta\ell r d\beta^2L^4
    }{\mu}.
\]

It remains to choose \(\eta\) and \(t\) so that \(\Delta_t\le\varepsilon\).
If
\[
    \eta
    \le
    \frac{\mu\varepsilon}
    {64\ell r d\beta^2L^4},
\]
then the bias term satisfies
\[
    \frac{
        32\eta\ell r d\beta^2L^4
    }{\mu}
    \le
    \frac{\varepsilon}{2}.
\]
If additionally
\[
    t
    \ge
    \frac{1}{\mu\eta}
    \log
    \frac{2\Delta_0}{\varepsilon},
\]
then
\[
    e^{-\mu\eta t}\Delta_0
    \le
    \frac{\varepsilon}{2}.
\]
Combining the two bounds yields
\[
    \mathbb{E}\bigl[
        f_\beta(\theta_t)-f_\beta^\star
    \bigr]
    =
    \Delta_t
    \le
    \varepsilon .
\]
This proves the theorem.
\end{proof}

\subsection{Proof of One-Step Robustness Improvement}
\label{app:proof-one-step-robustness}

We next prove that one ZO refinement step can reduce the perturbation robustness gap around an FO-aligned solution.
Recall the robustness gap
\[
    \mathsf{Rob}_{\rho}(\theta)
    :=
    f_\rho(\theta)-f(\theta),
    \qquad
    f_\rho(\theta)
    :=
    \mathbb{E}_{v\sim\mathcal{N}(0,I_d)}
    [
        f(\theta+\rho v)
    ].
\]

\begin{proof}[Proof of Theorem~\ref{thm:strict-zo-robustness}]
Let
\[
    \theta_0 := \theta^{\mathrm{fo}},
    \qquad
    g_\rho := \nabla f_\rho(\theta_0),
    \qquad
    G_\rho := \|g_\rho\|.
\]
Let \(\widehat g_0\) be the ZO estimator of \(\nabla f_\rho(\theta_0)\), and let
\[
    \theta_1
    =
    \theta_0-\eta\widehat g_0
\]
be the one-step ZO-refined parameter.
Assume
\[
    \mathbb{E}[\widehat g_0]=g_\rho,
    \qquad
    \mathbb{E}\bigl[
        \|\widehat g_0-g_\rho\|^2
    \bigr]
    \le
    \sigma_{\rm zo}^2 .
\]
Define
\[
    V
    :=
    \mathbb{E}[\|\widehat g_0\|^2].
\]
By variance decomposition,
\[
    V
    =
    \|g_\rho\|^2
    +
    \mathbb{E}[
        \|\widehat g_0-g_\rho\|^2
    ]
    \le
    G_\rho^2+\sigma_{\rm zo}^2 .
\]

We first bound the change of the smoothed objective.
Since \(f_\rho\) is \(h\)-smooth,
\[
    f_\rho(\theta_1)
    \le
    f_\rho(\theta_0)
    -
    \eta g_\rho^\top\widehat g_0
    +
    \frac{h\eta^2}{2}
    \|\widehat g_0\|^2 .
\]
Taking expectation and using \(\mathbb{E}[\widehat g_0]=g_\rho\), we obtain
\[
    \mathbb{E}[f_\rho(\theta_1)]
    \le
    f_\rho(\theta_0)
    -
    \eta G_\rho^2
    +
    \frac{h\eta^2}{2}V .
\]

We next bound the change of the objective \(f\).
Since \(f\) is \(h\)-smooth and \(\nabla f(\theta_0)=0\), smoothness gives
\[
    f(\theta_1)
    \ge
    f(\theta_0)
    -
    \frac{L_f}{2}
    \|\theta_1-\theta_0\|^2 .
\]
Using \(\theta_1-\theta_0=-\eta\widehat g_0\) and taking expectation,
\[
    \mathbb{E}[f(\theta_1)]
    \ge
    f(\theta_0)
    -
    \frac{L_f\eta^2}{2}V .
\]

Combining the two inequalities, we get
\[
\begin{aligned}
    \mathbb{E}[
        \mathsf{Rob}_{\rho}(\theta_1)
    ]
    &=
    \mathbb{E}[
        f_\rho(\theta_1)-f(\theta_1)
    ]  \\
    &\le
    f_\rho(\theta_0)-f(\theta_0)
    -
    \eta G_\rho^2
    +
    \eta^2hV  \\
    &=
    \mathsf{Rob}_{\rho}(\theta_0)
    -
    \eta G_\rho^2
    +
    \eta^2hV  .
\end{aligned}
\]
Therefore,
\[
    \mathbb{E}[
        \mathsf{Rob}_{\rho}(\theta_1)
    ]
    <
    \mathsf{Rob}_{\rho}(\theta_0)
\]
whenever
\[
    \eta G_\rho^2
    >
    \eta^2hV  .
\]
Equivalently, for any
\[
    0<\eta
    <
    \frac{
        G_\rho^2
    }{
        hV
    },
\]
the one-step ZO update strictly reduces the expected robustness gap.

Finally, using the bound
\[
    V
    \le
    G_\rho^2+\sigma_{\rm zo}^2,
\]
a sufficient stepsize condition is
\[
    0<\eta
    <
    \frac{
        G_\rho^2
    }{
        h
        (G_\rho^2+\sigma_{\rm zo}^2)
    }.
\]
This proves the claim.
\end{proof}

\clearpage
\section{Additional Experimental Results}\label{app:exps}

\begin{table}[h]
\centering
\small
\setlength{\tabcolsep}{4.2pt}
\renewcommand{\arraystretch}{1.16}
\caption{
Comparison between 100-step SFT and the subsequent 10-step ZO robustness refinement.
We keep the original model as the reference, and omit perturbed results for the original model on Llama3-8B-Instruct.
}
\label{tab:sft_zo_robustness}
\resizebox{\linewidth}{!}{
\begin{tabular}{llcccccc}
\toprule
\textbf{Method} 
& \textbf{Perturbation}
& \textbf{Level}
& \textbf{PPL}
& \textbf{lm-eval}
& \textbf{HarmBench ASR}
& \textbf{LlamaGuard3 ASR}
& \textbf{AdvBench ASR} \\
\midrule

Original 
& --
& -- 
& 8.2823 
& 0.6777 
& 0.1067 
& 0.0833 
& 0.1333 \\

\midrule
\multicolumn{8}{l}{\textbf{After 100-step SFT}} \\
\midrule

SFT 
& --
& -- 
& 9.4069 
& 0.6653 
& 0.0667 
& 0.0567 
& 0.0733  \\

SFT 
& Quantization 
& W4A16 
& 10.3574
& 0.6272
& 0.0900
& 0.0833
& 0.1367 \\

SFT 
& Quantization 
& W4A4 
& 15.0531
& 0.5445
& 0.1467
& 0.2500
& 0.3167 \\

SFT 
& Weight noise 
& 2 
& 21.7101
& 0.6023
& 0.1200
& 0.1067
& 0.1467 \\

SFT 
& Weight noise 
& 3 
& 63.5586
& 0.5569
& 0.1133
& 0.3800
& 0.3233 \\

SFT 
& Activation noise 
& 0.05 
& 10.6832
& 0.6375
& 0.0889
& 0.0744
& 0.0889 \\

SFT 
& Activation noise 
& 0.08 
& 13.3357
& 0.5991
& 0.1533
& 0.1422
& 0.1667 \\

SFT 
& Activation noise 
& 0.10 
& 16.7568
& 0.5664
& 0.2111
& 0.2600
& 0.2644 \\

SFT 
& Activation noise 
& 0.12 
& 22.7916
& 0.5183
& 0.2078
& 0.5922
& 0.4400 \\

\midrule
\multicolumn{8}{l}{\textbf{After 100-step SFT + 10-step ZO refinement}} \\
\midrule

SFT + ZO 
& -- 
& -- 
& 9.4067 {\scriptsize(\lbetter{0.0002})}
& 0.6651 {\scriptsize(\hworse{0.0002})}
& 0.0667 {\scriptsize(\same)}
& 0.0567 {\scriptsize(\same)}
& 0.0733 {\scriptsize(\same)} \\

SFT + ZO 
& Quantization 
& W4A16 
& 10.3580 {\scriptsize(\lworse{0.0006})}
& 0.6303 {\scriptsize(\hbetter{0.0031})}
& 0.0667 {\scriptsize(\lbetter{0.0233})}
& 0.0567 {\scriptsize(\lbetter{0.0266})}
& 0.0833 {\scriptsize(\lbetter{0.0534})} \\

SFT + ZO 
& Quantization 
& W4A4 
& 14.6690 {\scriptsize(\lbetter{0.3841})}
& 0.5483 {\scriptsize(\hbetter{0.0038})}
& 0.1400 {\scriptsize(\lbetter{0.0067})}
& 0.1967 {\scriptsize(\lbetter{0.0533})}
& 0.3333 {\scriptsize(\lworse{0.0166})} \\

SFT + ZO 
& Weight noise 
& $\sigma=2$ 
& 21.7099 {\scriptsize(\lbetter{0.0002})}
& 0.6023 {\scriptsize(\same)}
& 0.1200 {\scriptsize(\same)}
& 0.1067 {\scriptsize(\same)}
& 0.1467 {\scriptsize(\same)} \\

SFT + ZO 
& Weight noise 
& $\sigma=3$ 
& 63.5571 {\scriptsize(\lbetter{0.0015})}
& 0.5573 {\scriptsize(\hbetter{0.0004})}
& 0.1100 {\scriptsize(\lbetter{0.0033})}
& 0.3800 {\scriptsize(\same)}
& 0.3200 {\scriptsize(\lbetter{0.0033})} \\

SFT + ZO 
& Activation noise 
& $\alpha=0.05$ 
& 10.6830 {\scriptsize(\lbetter{0.0002})}
& 0.6373 {\scriptsize(\hworse{0.0002})}
& 0.0789 {\scriptsize(\lbetter{0.0100})}
& 0.0656 {\scriptsize(\lbetter{0.0088})}
& 0.0822 {\scriptsize(\lbetter{0.0067})} \\

SFT + ZO 
& Activation noise 
& $\alpha=0.08$ 
& 13.3357 {\scriptsize(\same)}
& 0.5991 {\scriptsize(\same)}
& 0.1444 {\scriptsize(\lbetter{0.0089})}
& 0.1322 {\scriptsize(\lbetter{0.0100})}
& 0.1567 {\scriptsize(\lbetter{0.0100})} \\

SFT + ZO 
& Activation noise 
& $\alpha=0.10$ 
& 16.7565 {\scriptsize(\lbetter{0.0003})}
& 0.5662 {\scriptsize(\hworse{0.0002})}
& 0.1967 {\scriptsize(\lbetter{0.0144})}
& 0.2511 {\scriptsize(\lbetter{0.0089})}
& 0.2489 {\scriptsize(\lbetter{0.0155})} \\

SFT + ZO 
& Activation noise 
& $\alpha=0.12$ 
& 22.7921 {\scriptsize(\lworse{0.0005})}
& 0.5183 {\scriptsize(\same)}
& 0.1967 {\scriptsize(\lbetter{0.0111})}
& 0.5667 {\scriptsize(\lbetter{0.0255})}
& 0.4144 {\scriptsize(\lbetter{0.0256})} \\

\bottomrule
\end{tabular}
}
\end{table}

\begin{table}[t]
\centering
\small
\setlength{\tabcolsep}{4.2pt}
\renewcommand{\arraystretch}{1.12}
\newcommand{\best}[1]{\textbf{#1}}
\caption{
Complete results of 10-step ZO robustness refinement with different layer-selection strategies on Llama3-8B-Instruct.
SNIP updates layers 15, 16, 14, and 13; WANDA updates layers 31, 30, 29, and 28; robustness-aware selection updates layers 24, 12, 31, and 29.
Bold numbers indicate the best ASR among the three layer-selection strategies under the same perturbation setting.
}

\label{tab:zo_pruning_full_results}
\resizebox{\linewidth}{!}{
\begin{tabular}{llcccccc}
\toprule
\textbf{Method} 
& \textbf{Perturbation}
& \textbf{Level}
& \textbf{PPL}
& \textbf{lm-eval}
& \textbf{HarmBench ASR}
& \textbf{LlamaGuard3 ASR}
& \textbf{AdvBench ASR} \\
\midrule
\multicolumn{8}{l}{\textbf{ZO refinement with SNIP-selected layers}} \\
\midrule
SNIP + ZO & -- & -- & 9.4227 & 0.6656 & \best{0.0700} & \best{0.0567} & \best{0.0767} \\
SNIP + ZO & Quantization & W4A16 & 10.4175 & 0.6230 & 0.1133 & 0.0933 & 0.1267 \\
SNIP + ZO & Quantization & W4A4 & 14.5035 & 0.5428 & 0.1667 & 0.2333 & 0.3167 \\
SNIP + ZO & Weight noise & $\sigma=1$ & 11.5789 & 0.6500 & \best{0.0900} & \best{0.0667} & \best{0.1067} \\
SNIP + ZO & Weight noise & $\sigma=1.5$ & 14.9834 & 0.6306 & 0.0933 & \best{0.0767} & \best{0.1267} \\
SNIP + ZO & Weight noise & $\sigma=2$ & 21.5331 & 0.6139 & \best{0.1067} & 0.0967 & \best{0.1733} \\
SNIP + ZO & Activation noise & $\alpha=0.05$ & 10.6951 & 0.6358 & 0.0867 & \best{0.0700} & 0.1000 \\
SNIP + ZO & Activation noise & $\alpha=0.08$ & 13.3717 & 0.6037 & 0.1567 & 0.1400 & 0.1800 \\
SNIP + ZO & Activation noise & $\alpha=0.10$ & 16.8277 & 0.5680 & 0.2167 & 0.2633 & 0.2500 \\
SNIP + ZO & Activation noise & $\alpha=0.12$ & 22.9064 & 0.5135 & 0.2233 & 0.6133 & 0.4333 \\

\midrule
\multicolumn{8}{l}{\textbf{ZO refinement with WANDA-selected layers}} \\
\midrule
WANDA + ZO & -- & -- & 9.4228 & 0.6655 & \best{0.0700} & \best{0.0567} & \best{0.0767} \\
WANDA + ZO & Quantization & W4A16 & 10.4033 & 0.6239 & 0.1367 & 0.1200 & 0.1600 \\
WANDA + ZO & Quantization & W4A4 & 14.6934 & 0.5448 & 0.1933 & 0.2633 & 0.3733 \\
WANDA + ZO & Weight noise & $\sigma=1$ & 11.5790 & 0.6498 & \best{0.0900} & \best{0.0667} & \best{0.1067} \\
WANDA + ZO & Weight noise & $\sigma=1.5$ & 14.9836 & 0.6309 & 0.0933 & \best{0.0767} & \best{0.1267} \\
WANDA + ZO & Weight noise & $\sigma=2$ & 21.5335 & 0.6139 & \best{0.1067} & \best{0.0933} & \best{0.1733} \\
WANDA + ZO & Activation noise & $\alpha=0.05$ & 10.6951 & 0.6357 & 0.0967 & 0.0767 & \best{0.0900} \\
WANDA + ZO & Activation noise & $\alpha=0.08$ & 13.3718 & 0.6039 & 0.1533 & 0.1400 & 0.1800 \\
WANDA + ZO & Activation noise & $\alpha=0.10$ & 16.8279 & 0.5680 & 0.2167 & 0.2633 & 0.2500 \\
WANDA + ZO & Activation noise & $\alpha=0.12$ & 22.9064 & 0.5135 & 0.2233 & 0.6133 & 0.4333 \\

\midrule
\multicolumn{8}{l}{\textbf{ZO refinement with robustness-aware selected layers}} \\
\midrule
Robust + ZO & -- & -- & 9.4229 & 0.6655 & \best{0.0700} & \best{0.0567} & \best{0.0767} \\
Robust + ZO & Quantization & W4A16 & 10.4091 & 0.6220 & \best{0.0900} & \best{0.0733} & \best{0.1000} \\
Robust + ZO & Quantization & W4A4 & 14.4467 & 0.5440 & \best{0.1233} & \best{0.1767} & \best{0.2733} \\
Robust + ZO & Weight noise & $\sigma=1$ & 11.5791 & 0.6497 & \best{0.0900} & \best{0.0667} & \best{0.1067} \\
Robust + ZO & Weight noise & $\sigma=1.5$ & 14.9837 & 0.6306 & \best{0.0900} & \best{0.0767} & \best{0.1267} \\
Robust + ZO & Weight noise & $\sigma=2$ & 21.5337 & 0.6137 & \best{0.1067} & 0.0967 & \best{0.1733} \\
Robust + ZO & Activation noise & $\alpha=0.05$ & 10.6951 & 0.6358 & \best{0.0833} & 0.0767 & \best{0.0900} \\
Robust + ZO & Activation noise & $\alpha=0.08$ & 13.3717 & 0.6039 & \best{0.1300} & \best{0.1367} & \best{0.1567} \\
Robust + ZO & Activation noise & $\alpha=0.10$ & 16.8278 & 0.5683 & \best{0.2033} & \best{0.2367} & \best{0.2333} \\
Robust + ZO & Activation noise & $\alpha=0.12$ & 22.9063 & 0.5136 & \best{0.2167} & \best{0.5600} & \best{0.4300} \\

\bottomrule
\end{tabular}
}
\end{table}

\begin{table}[t]
\centering
\small
\setlength{\tabcolsep}{4.2pt}
\renewcommand{\arraystretch}{1.16}
\caption{
Comparison between the original Qwen2-7B-Instruct model, 100-step SFT, and the subsequent 10-step ZO robustness refinement.
}
\label{tab:qwen_sft_zo_robustness}
\resizebox{\linewidth}{!}{
\begin{tabular}{llcccccc}
\toprule
\textbf{Method} 
& \textbf{Perturbation}
& \textbf{Level}
& \textbf{PPL}
& \textbf{lm-eval}
& \textbf{HarmBench ASR}
& \textbf{LlamaGuard3 ASR}
& \textbf{AdvBench ASR} \\
\midrule
\multicolumn{8}{l}{\textbf{Original}} \\
\midrule

Original & -- & -- & 7.5994 & 0.6840 & 0.2967 & 0.3067 & 0.3333 \\
Original & Quantization & W4A16 & 7.9789 & 0.6623 & 0.2767 & 0.2833 & 0.3333 \\
Original & Quantization & W4A4 & 3231.7378 & 0.3134 & 0.0000 & 0.7233 & 1.0000 \\

Original & Weight noise & $\sigma=1$ & 9.9568 & 0.6501 & 0.3144 & 0.3267 & 0.3311 \\
Original & Weight noise & $\sigma=1.5$ & 13.9043 & 0.6225 & 0.3022 & 0.3378 & 0.3589 \\
Original & Weight noise & $\sigma=2$ & 22.1413 & 0.5959 & 0.2489 & 0.3900 & 0.3989 \\

Original & Activation noise & $\alpha=0.05$ & 7.8890 & 0.6657 & 0.3017 & 0.3150 & 0.3333 \\
Original & Activation noise & $\alpha=0.10$ & 8.7442 & 0.6318 & 0.3042 & 0.3258 & 0.3467 \\
Original & Activation noise & $\alpha=0.15$ & 10.4749 & 0.5697 & 0.3389 & 0.3989 & 0.4278 \\

\midrule
\multicolumn{8}{l}{\textbf{After 100-step SFT}} \\
\midrule

SFT & -- & -- & 7.6408 & 0.6796 & 0.1100 & 0.0933 & 0.1033 \\
SFT & Quantization & W4A16 & 8.0433 & 0.6704 & 0.0867 & 0.0833 & 0.1167 \\
SFT & Quantization & W4A4 & 3801.4229 & 0.3037 & 0.0033 & 0.8667 & 1.0000 \\

SFT & Weight noise & 1 & 10.0349 & 0.6452 & 0.1222 & 0.1033 & 0.1467 \\
SFT & Weight noise & 1.5 & 14.0629 & 0.6162 & 0.1333 & 0.1522 & 0.1978 \\
SFT & Weight noise & 2 & 22.5278 & 0.5851 & 0.1156 & 0.2522 & 0.2789 \\

SFT & Activation noise & 0.05 & 7.9293 & 0.6643 & 0.1056 & 0.0900 & 0.1022 \\
SFT & Activation noise & 0.10 & 8.7775 & 0.6336 & 0.0989 & 0.0956 & 0.1167 \\
SFT & Activation noise & 0.15 & 10.4959 & 0.5740 & 0.1389 & 0.1856 & 0.1967 \\

\midrule
\multicolumn{8}{l}{\textbf{After 100-step SFT + 10-step ZO refinement}} \\
\midrule

SFT + ZO 
& -- 
& -- 
& 7.6408 {\scriptsize(\same)}
& 0.6795 {\scriptsize(\hworse{0.0001})}
& 0.1100 {\scriptsize(\same)}
& 0.0933 {\scriptsize(\same)}
& 0.1033 {\scriptsize(\same)} \\

SFT + ZO 
& Quantization 
& W4A16 
& 8.0532 {\scriptsize(\lworse{0.0099})}
& 0.6703 {\scriptsize(\hworse{0.0001})}
& 0.0767 {\scriptsize(\lbetter{0.0100})}
& 0.0767 {\scriptsize(\lbetter{0.0067})}
& 0.1000 {\scriptsize(\lbetter{0.0167})} \\

SFT + ZO 
& Quantization 
& W4A4 
& 2887.3979 {\scriptsize(\lbetter{914.0250})}
& 0.3151 {\scriptsize(\hbetter{0.0114})}
& 0.0033 {\scriptsize(\same)}
& 0.7933 {\scriptsize(\lbetter{0.0733})}
& 1.0000 {\scriptsize(\same)} \\

SFT + ZO 
& Weight noise 
& $\sigma=1$ 
& 10.0349 {\scriptsize(\same)}
& 0.6456 {\scriptsize(\hbetter{0.0004})}
& 0.1222 {\scriptsize(\same)}
& 0.1078 {\scriptsize(\lworse{0.0044})}
& 0.1500 {\scriptsize(\lworse{0.0033})} \\

SFT + ZO 
& Weight noise 
& $\sigma=1.5$ 
& 14.0629 {\scriptsize(\same)}
& 0.6160 {\scriptsize(\hworse{0.0002})}
& 0.1389 {\scriptsize(\lworse{0.0056})}
& 0.1511 {\scriptsize(\lbetter{0.0011})}
& 0.1956 {\scriptsize(\lbetter{0.0022})} \\

SFT + ZO 
& Weight noise 
& $\sigma=2$ 
& 22.5278 {\scriptsize(\same)}
& 0.5849 {\scriptsize(\hworse{0.0002})}
& 0.1111 {\scriptsize(\lbetter{0.0044})}
& 0.2489 {\scriptsize(\lbetter{0.0033})}
& 0.2800 {\scriptsize(\lworse{0.0011})} \\

SFT + ZO 
& Activation noise 
& $\alpha=0.05$ 
& 7.9293 {\scriptsize(\same)}
& 0.6646 {\scriptsize(\hbetter{0.0003})}
& 0.0956 {\scriptsize(\lbetter{0.0100})}
& 0.0789 {\scriptsize(\lbetter{0.0111})}
& 0.0978 {\scriptsize(\lbetter{0.0044})} \\

SFT + ZO 
& Activation noise 
& $\alpha=0.10$ 
& 8.7776 {\scriptsize(\lworse{0.0001})}
& 0.6336 {\scriptsize(\same)}
& 0.0978 {\scriptsize(\lbetter{0.0011})}
& 0.0989 {\scriptsize(\lworse{0.0033})}
& 0.1200 {\scriptsize(\lworse{0.0033})} \\

SFT + ZO 
& Activation noise 
& $\alpha=0.15$ 
& 10.4959 {\scriptsize(\same)}
& 0.5739 {\scriptsize(\hworse{0.0001})}
& 0.1167 {\scriptsize(\lbetter{0.0222})}
& 0.1533 {\scriptsize(\lbetter{0.0322})}
& 0.1800 {\scriptsize(\lbetter{0.0167})} \\

\bottomrule
\end{tabular}
}
\end{table}

\begin{table}[t]
\centering
\small
\setlength{\tabcolsep}{4.2pt}
\renewcommand{\arraystretch}{1.12}
\newcommand{\best}[1]{\textbf{#1}}
\caption{
Complete results of 10-step ZO robustness refinement with different layer-selection strategies on Qwen2-7B-Instruct.
SNIP updates layers 18, 19, and 17; WANDA updates layers 27, 26, and 25; robustness-aware selection updates layers 12, 19, and 11.
Bold numbers indicate the best ASR among the three layer-selection strategies under the same perturbation setting.
}
\label{tab:zo_pruning_lr2em2_mu5em3_layers12_19_11}
\resizebox{\linewidth}{!}{
\begin{tabular}{llcccccc}
\toprule
\textbf{Method}
& \textbf{Perturbation}
& \textbf{Level}
& \textbf{PPL}
& \textbf{lm-eval}
& \textbf{HarmBench ASR}
& \textbf{LlamaGuard3 ASR}
& \textbf{AdvBench ASR} \\
\midrule
\multicolumn{8}{l}{\textbf{ZO refinement with SNIP-selected layers}} \\
\midrule
SNIP + ZO & -- & -- & 7.6408 & 0.6794 & 0.1100 & 0.0933 & \best{0.1033} \\
SNIP + ZO & Quantization & W4A16 & 8.0558 & 0.6713 & 0.1000 & 0.0900 & 0.1000 \\
SNIP + ZO & Weight noise & $\sigma=1$ & 8.1373 & 0.6675 & 0.1122 & 0.0922 & 0.1122 \\
SNIP + ZO & Weight noise & $\sigma=2$ & 12.3782 & 0.5992 & 0.1356 & 0.1767 & 0.1500 \\
SNIP + ZO & Weight noise & $\sigma=3$ & 16.7780 & 0.5509 & 0.1367 & 0.2811 & \best{0.2000} \\
SNIP + ZO & Activation noise & $\alpha=0.05$ & 7.9293 & 0.6647 & 0.1100 & 0.0944 & 0.1000 \\
SNIP + ZO & Activation noise & $\alpha=0.10$ & 8.7776 & 0.6336 & 0.1067 & 0.1111 & 0.1278 \\
SNIP + ZO & Activation noise & $\alpha=0.15$ & 10.4961 & 0.5739 & 0.1333 & 0.1811 & 0.1833 \\

\midrule
\multicolumn{8}{l}{\textbf{ZO refinement with WANDA-selected layers}} \\
\midrule
WANDA + ZO & -- & -- & 7.6408 & 0.6797 & 0.1100 & \best{0.0900} & \best{0.1033} \\
WANDA + ZO & Quantization & W4A16 & 8.0439 & 0.6731 & 0.1000 & 0.0767 & 0.1100 \\
WANDA + ZO & Weight noise & $\sigma=1$ & 8.1374 & 0.6676 & 0.1122 & 0.0922 & 0.1111 \\
WANDA + ZO & Weight noise & $\sigma=2$ & 12.3782 & 0.5992 & 0.1344 & 0.1778 & \best{0.1467} \\
WANDA + ZO & Weight noise & $\sigma=3$ & 16.7778 & 0.5513 & 0.1300 & 0.2767 & 0.2022 \\
WANDA + ZO & Activation noise & $\alpha=0.05$ & 7.9293 & 0.6646 & 0.1100 & 0.0978 & 0.1056 \\
WANDA + ZO & Activation noise & $\alpha=0.10$ & 8.7775 & 0.6336 & \best{0.0956} & 0.1056 & 0.1256 \\
WANDA + ZO & Activation noise & $\alpha=0.15$ & 10.4960 & 0.5740 & 0.1311 & \best{0.1744} & \best{0.1800} \\

\midrule
\multicolumn{8}{l}{\textbf{ZO refinement with robustness-aware selected layers}} \\
\midrule
Robust + ZO & -- & -- & 7.6408 & 0.6798 & \best{0.1067} & \best{0.0900} & \best{0.1033} \\
Robust + ZO & Quantization & W4A16 & 8.0485 & 0.6707 & \best{0.0833} & \best{0.0733} & \best{0.0800} \\
Robust + ZO & Weight noise & $\sigma=1$ & 8.1374 & 0.6677 & \best{0.1111} & \best{0.0900} & \best{0.1044} \\
Robust + ZO & Weight noise & $\sigma=2$ & 12.3783 & 0.5993 & \best{0.1311} & \best{0.1744} & \best{0.1467} \\
Robust + ZO & Weight noise & $\sigma=3$ & 16.7780 & 0.5513 & \best{0.1278} & \best{0.2756} & 0.2011 \\
Robust + ZO & Activation noise & $\alpha=0.05$ & 7.9293 & 0.6644 & \best{0.0956} & \best{0.0733} & \best{0.0944} \\
Robust + ZO & Activation noise & $\alpha=0.10$ & 8.7775 & 0.6334 & 0.1000 & \best{0.0978} & \best{0.1200} \\
Robust + ZO & Activation noise & $\alpha=0.15$ & 10.4960 & 0.5739 & \best{0.1178} & 0.1889 & 0.1856 \\

\bottomrule
\end{tabular}
}
\end{table}

\end{document}